\documentclass[10pt]{article}

\usepackage[letterpaper,top=0.68in,bottom=0.76in,left=0.75in,right=0.75in]{geometry}
\usepackage[T1]{fontenc}
\usepackage{lmodern}
\usepackage{amsmath,amssymb,amsthm,mathtools}
\usepackage{graphicx}
\usepackage{tikz}
\usetikzlibrary{cd,arrows.meta,positioning,fit,backgrounds}
\usepackage{iftex}
\ifPDFTeX
\usepackage[protrusion=true,expansion=true]{microtype}
\else
\usepackage[protrusion=true,expansion=false]{microtype}
\fi
\usepackage{enumitem}
\usepackage{booktabs}
\usepackage{array}
\usepackage{xcolor}
\usepackage{soul}
\soulregister\citep7
\soulregister\citet7
\soulregister\cite7
\soulregister\ref7
\usepackage{pdfpages}

\newcommand{\hlyellow}[1]{{\sethlcolor{white} \hl{#1}}}
\newcommand{\hlrev}[1]{{\sethlcolor{white} \hl{#1}}}

\usepackage[colorlinks=true,linkcolor=blue!50!black,citecolor=blue!50!black,urlcolor=blue!50!black]{hyperref}
\usepackage[numbers,sort&compress]{natbib}
\usepackage{titlesec}
\usepackage{abstract}
\usepackage{placeins}
\usepackage{titling}
\usepackage[font={footnotesize,sf},labelfont={bf,footnotesize,sf},textfont={footnotesize}]{caption}
\usepackage{orcidlink}

\AtBeginDocument{%
  \setlength{\abovedisplayskip}{10pt plus 2pt minus 2pt}%
  \setlength{\belowdisplayskip}{10pt plus 2pt minus 2pt}%
  \setlength{\abovedisplayshortskip}{7pt plus 2pt minus 2pt}%
  \setlength{\belowdisplayshortskip}{7pt plus 2pt minus 2pt}%
}
\setlist{nosep,leftmargin=*}
\titleformat{\section}{\sffamily\bfseries\large}{\thesection.}{0.55em}{}
\titleformat{\subsection}{\sffamily\bfseries\normalsize}{\thesubsection.}{0.55em}{}
\titleformat{\subsubsection}{\sffamily\bfseries\normalsize}{\thesubsubsection.}{0.55em}{}
\titlespacing*{\section}{0pt}{2ex plus 0.5ex minus 0.2ex}{1ex}
\titlespacing*{\subsection}{0pt}{1.5ex plus 0.3ex minus 0.2ex}{0.75ex}
\titlespacing*{\subsubsection}{0pt}{0.75ex plus 0.25ex minus 0.15ex}{0.35ex}

\pretitle{\begin{center}\sffamily\bfseries\Large}
\posttitle{\par\end{center}\vspace{-0.75em}}

\preauthor{\begin{center}\sffamily\small}
\postauthor{\par\end{center}\vspace{-1.15em}}
\predate{}
\postdate{}

\newtheorem{theorem}{Theorem}
\newtheorem{proposition}[theorem]{Proposition}

\theoremstyle{definition}
\newtheorem{definition}{Definition}
\newtheorem{observation}[theorem]{Observation}
\newtheorem{remark}{Remark}
\newtheorem*{openproblem}{Open Problem}

\newcommand{\K}{\mathcal{K}}
\newcommand{\Sschema}{\mathcal{S}}
\newcommand{\V}{\mathcal{V}}
\newcommand{\Set}{\mathbf{Set}}
\newcommand{\Obj}{\operatorname{Obj}}
\newcommand{\Lan}{\operatorname{Lan}}
\newcommand{\colim}{\operatorname{colim}}
\newcommand{\N}{\mathbb{N}}
\newcommand{\Prov}{\mathsf{Prov}}

\title{\textbf{
Self-Revising Discovery Systems for Science: A Categorical Framework for Agentic Artificial Intelligence}}

\author{%
\begin{tabular}{c}
Fiona Y. Wang\textsuperscript{1,2} \quad Markus J. Buehler\textsuperscript{2,3,4*}\\[-0.1em]
\footnotesize \textsuperscript{1}Laboratory for Atomistic and Molecular Mechanics, MIT\\[-0.15em]
\footnotesize \textsuperscript{2}Department of Biological Engineering, MIT\\[-0.15em]
\footnotesize \textsuperscript{3}Departments of Civil and Environmental Engineering and Mechanical Engineering,\\[-0.15em]
\footnotesize \textsuperscript{4}Center for Computational Science and Engineering, Schwarzman College of Computing, MIT\\[-0.15em]
\footnotesize Cambridge, MA 02139, USA; \textsuperscript{*}Corresponding author: \texttt{mbuehler@mit.edu}
\end{tabular}
}
\date{}

\begin{document}
\maketitle

\begin{abstract}
\noindent
\hlyellow{Scientific discovery is not merely the production of better answers, but the crossing of a threshold into a new representational regime in which previously unnameable evidence becomes expressible, testable, and reusable. The central difficulty for agentic artificial intelligence is that probabilistic prediction and optimization operate within a chosen hypothesis space: they can rank alternatives once the admissible variables, artifacts, and tests are fixed, but they do not by themselves identify when those representational commitments have become inadequate. Discovery often occurs precisely at this boundary, when new objects, morphisms, descriptors, or verifiers must be introduced for the evidence to become expressible.}
\hlyellow{We develop a category-theoretic formalism for this threshold. In a fixed regime with schema category $\Sschema_b$, the committed scientific state is a copresheaf $I_t:\Sschema_b\to\Set$, with realized provenance represented by the category of elements $\int_{\Sschema_b}I_t$ or by a multicategorical analogue for multi-parent artifacts. Search updates this state within the same schema. Discovery is a verified transition $u:\Sschema_b\to\Sschema_{b^{\prime}}$ that preserves accepted prior artifacts, transports old evidence by left Kan extension, and exposes residual content not generated by transport. The formalism therefore characterizes discovery as a change in admissible representation, rather than as a subjective novelty judgment or a higher model score.}
\hlyellow{We instantiate the framework in Builder/Breaker and CategoryScienceClaw. Builder/Breaker revises a protein-mechanics symbolic world model under a minimum description length gate, retaining rejected edits and retractions; the accepted law expresses within-chain flexibility as all-mode elastic compliance conditioned by slow collective-mode participation. CategoryScienceClaw extends ScienceClaw with proof-carrying typed artifacts, lineage, gates, stress tests, and claim records. Its high-entropy-alloy fatigue case rejects threshold-only $\Delta K_{th}$ screening and records a two-axis damage-tolerance representation separating crack-growth threshold from Paris-slope stability. The resulting claim is a reproducible public-data hypothesis and prospective validation target, not a substitute for new fatigue testing. Together, these cases show that categorical structure can make representational novelty explicit, auditable, and operational in AI-driven scientific discovery.}
\end{abstract}


% =====================================================================
\section{Introduction}
% =====================================================================

Artificial intelligence is now embedded in most stages of the scientific process. Foundation models retrieve and summarize literature, propose hypotheses, write and debug code, run and interpret simulations, design proteins and materials, and draft figures and reports. Agentic systems built on top of these models call external tools, coordinate multiple specialized subsystems, manage long-running workflows, and increasingly take partial responsibility for experimental decisions, both in computational pipelines and in autonomous laboratories \citep{doi:10.1021/acsengchemau.3c00053,ni_mechagents_2024,ghafarollahi_protagents_2024,ghafarollahi_sparks_2025,ghafarollahi_sciagents_2025,lu_ai_scientist_2024,lu_ai_scientist_v2_2025,agarwal_autodiscovery_2025,wang_scienceclaw_infinite_2026,buehler_break_world_2026,buehler_melm_2023,buehler_preflexor_2025}. The trajectory is clear and suggests that AI is no longer only a way to predict an output for a fixed task, but an active participant in how scientific work is structured.

Yet the central question for these systems remains underformalized. Existing AI scientists are extraordinarily fluent at recombining, optimizing, and reformulating inside a fixed scientific vocabulary, but the operations that matter most in real science often change the vocabulary itself: a new effective variable, a new admissible operation, a new verifier, a new tool, a new artifact type. When is an agentic system searching within a fixed scientific regime, and when is it changing the regime itself? The answer is not merely philosophical; it determines how verifiers must be designed, how provenance must be audited, how progress should be measured, and why scaling a fixed model is qualitatively different from building a system that can construct new representational commitments.

\hlyellow{This distinction leads to the central technical problem of the paper: identifying when the current hypothesis space is no longer an adequate representation of the scientific problem. Probabilistic scoring can compare hypotheses within a specified representation, but such scores do not by themselves determine when new variables, artifact types, operations, verifiers, or evidentiary relations must be introduced. In the terminology of this paper, discovery begins at this representational boundary. We formalize that boundary by representing scientific state as typed commitments over a schema and representing discovery as a recorded regime transition $u:\Sschema_b\to\Sschema_{b'}$. Prior evidence is preserved and transported where possible; accepted and rejected commitments remain auditable; and the new representational content is identified by residual objects that are not generated by transport from the old schema.}

\hlyellow{The problem is connected to a longer account of scientific change. Popper emphasized critical tests and refutation; Kuhn emphasized changes of paradigm and world view; Lakatos described scientific progress through research programmes whose hard cores and auxiliary hypotheses evolve under pressure from anomalies \citep{popper_logic_1959,kuhn_structure_1962,lakatos_falsification_1970}. We translate this philosophical distinction into a design requirement for artificial discovery systems: the system must record, verify, and reuse the moment when evidence forces a change in the representational regime.}

\hlyellow{This requirement sits at the intersection of several traditions that have usually been developed separately. Scientific knowledge graphs and scholarly graph infrastructures organize entities, relations, and claims \citep{jaradeh2019orkg,doi:10.1021/acsengchemau.3c00053,stewart2026graphagentsknowledgegraphguidedagentic}; provenance models and scientific-workflow systems record how data products are generated, transformed, and reused \citep{provo2013,davidson2008provenance,bechhofer2013linkeddata}; and categorical databases and applied category theory provide a language for schemas, typed data, and functorial migration between representations \citep{spivak_schemas_2012,spivak_categories_scientists_2014,fong_spivak_seven_2019}. In mechanics and materials science, language-model and multi-agent systems have begun to connect knowledge across scales and support exploratory scientific reasoning \citep{buehler_melm_2023,buehler_mechgpt_2024,ni_mechagents_2024,ghafarollahi_protagents_2024,ghafarollahi_atomagents_2025,ghafarollahi_sciagents_2025,buehler_preflexor_2025}.}

\hlyellow{Our contribution is to bring these strands into a single account of scientific state. A discovery system is treated not only as a graph of artifacts, a workflow trace, or an agentic controller, but as a typed record of commitments: what objects currently exist, how they were produced, which alternatives were rejected, which gates accepted them, and when the vocabulary used to judge evidence has changed. CategoryScienceClaw implements this view as a proof-carrying provenance layer in which agents, artifacts, rejected alternatives, gates, stress tests, public claims, and regime changes inhabit the same typed discovery trace.}

\hlyellow{A concrete mechanics workflow makes this requirement tangible. Consider a researcher studying the mechanical response of a protein. They begin with a sequence, retrieve or predict a Protein Data Bank (PDB)-style structure, construct a contact graph, diagonalize an elastic network, compare predicted fluctuations with crystallographic B-factors, formulate a hypothesis about which residues dominate the response, select another protein whose behavior should stress the hypothesis, and revise the model. The objects at every stage are typed: sequence, structure, contact graph, mode amplitude, feature, symbolic model, score, report. The operations are typed as well: build a contact graph from a structure, diagonalize a Kirchhoff matrix, extract a normal mode, fit a symbolic expression, score a description-length budget. The record of the work is therefore not a string of answers but a typed provenance graph whose scientific meaning depends on the representation in which those objects and operations are admissible.} Table~\ref{tab:dictionary} \hlyellow{summarizes the implementation-to-categorical dictionary used throughout the paper, and Supplementary Table~S1 provides a broader plain-language terminology guide covering mechanics, statistics, scientific-machine-learning, materials-science, provenance, and evidence-scope terms.}

\begin{table}[ht]
\centering
\footnotesize
\setlength{\tabcolsep}{4pt}
\begin{tabular}{@{}>{\raggedright\arraybackslash}p{0.22\linewidth}>{\raggedright\arraybackslash}p{0.33\linewidth}>{\raggedright\arraybackslash}p{0.35\linewidth}@{}}
\toprule
Implementation object & Categorical object & Role in discovery system \\
\midrule
artifact type & object of schema category $\Sschema_b$ & declares what kind of thing may be produced \\
tool or skill signature & morphism of $\Sschema_b$ & declares allowed transformation between types \\
current artifact population & copresheaf $I_t:\Sschema_b\to\Set$ & stores artifacts inhabiting each type \\
artifact directed acyclic graph (DAG) or hypergraph & category of elements $\int I_t$, or a multicategorical provenance analogue& realized provenance graph \\
accepted update & natural/refinement morphism & preserves prior typed provenance \\
gate or verifier & predicate or scoring functional & decides commitment, rejection, or supersession \\
new schema/tool/verifier & regime extension $u:\Sschema_b\to\Sschema_{b'}$ & changes the admissible scientific vocabulary \\
\bottomrule
\end{tabular}
\caption{Dictionary between implementation terms and the categorical formalism used in this paper.}
\label{tab:dictionary}
\end{table}

Now suppose the next protein exposes a failure that cannot be repaired by changing a coefficient or adding another threshold. A local elastic-network feature may fit compact proteins but fail on hinge/domain proteins because the relevant phenomenon is no longer only local residue compliance; it is compliance expressed through a collective deformation. The researcher has two options. They can search within the current vocabulary, adjusting terms already available. Or they can enlarge the vocabulary by introducing a new effective type, operation, or verifier. The first move is search. The second is discovery in the strong sense used here: not only a better point in an existing space, but a change in the space of admissible scientific artifacts.

Category theory is useful here because it gives names to the engineering structure scientists already use. A schema is a category of artifact types and allowed operations. A current body of evidence is a population of artifacts over that schema. A provenance graph is the realized category of elements of that population. A consistent update is a natural transformation or functorial refinement. A discovery move is a transport from one schema to a larger or different schema, preserving what was valid while making new types, morphisms, tools, or verifiers available. The categorical foundations are standard; the applied use of categories for scientific schemas, ologs, and data migration is already well developed \citep{maclane_categories_1971,awodey_category_2010,spivak_categories_scientists_2014,fong_spivak_seven_2019,spivak_schemas_2012,spivak2020poly,spivak2021learners}. The language is abstract, but the objects are practical: PDB chains, simulations, equations, hypotheses, artifacts, claims, and reports.

\hlyellow{In this usage, the categorical terms do not introduce a second description detached from scientific practice; they formalize the dependence of evidence on the regime in which it is expressed. The schema specifies the presently admissible objects and transformations, and the artifact population over that schema records what the system is prepared to treat as evidence. The category of elements then identifies the realized portion of that typed evidence as an actual provenance structure. A fixed schema supports refinement and search within an established regime. A change of schema has a different status: prior artifacts must be transported across the change, and whatever cannot be accounted for by that transport marks the representational content newly made available. The formalism is therefore used to state, in auditable terms, how scientific meaning is preserved through a revision of the vocabulary that supports it.}

This paper also continues a materials-science arc developed in our earlier work on ologs, hierarchical materials, learned maps, neural ologs, language-mediated reasoning, and plannerless scientific swarms \citep{buehler_spivak_2011,giesa_spivak_2011_patterns,giesa_spivak_2012_buildingblock,buehler_fieldperceiver_2022,buehler_atoms_to_swarms_2026}. In each case the scientific problem is not only to compute an output, but to preserve the structure that makes the output meaningful across scales. The present work turns that arc back onto discovery itself. If a material is a hierarchy of composable mechanisms, an agentic discovery system is a hierarchy of composable scientific artifacts.

We note that the underlying claim is not ordinary reductionism. In hierarchical materials, complexity is rarely located at a single privileged scale. Hydroxyapatite chemistry, collagen structure, mineral organization, crack-tip mechanics, and tissue remodeling are each necessary, but none is sufficient alone. What matters is the compositional grammar that organizes simple components into higher-order structure, and the responsiveness by which that structure updates under load, damage, growth, or new evidence. This view has predecessors in older scientific accounts of form and transformation, including Goethe's morphology \citep{goethe_metamorphosis_1790}. Category theory is useful because it describes the morphisms among scales without forcing the explanation to live only at the bottom. The same principle motivates the present account of discovery: a scientific AI system should not only optimize artifacts inside a fixed representation, but should compose typed artifacts across representational levels, test those compositions against the world, and revise the grammar when the old one is too small.

This lineage also reverses the usual direction of influence between artificial intelligence and mechanics.
In many current settings, AI is introduced into mechanics as an external optimizer or surrogate predictor: it accelerates
simulation, fits constitutive laws, or searches a design space. The present framework instead lets mechanics
help define what an AI discovery system should be. Mechanics supplies a disciplined language for state,
load, response, instability, failure, admissible motion, constitutive closure, and multiscale transfer. These
ideas reappear here as artifact states, evidence pressure, stress tests, gates, regime transitions, residual
content, and provenance-preserving transport. In this sense, mechanics is not only a domain on which the
framework is demonstrated but instead is one of the sources of the framework. The same habits that make a
mechanical model meaningful (e.g., tracking boundary conditions, preserving invariants, testing failure modes,
coarse-graining across scales, and distinguishing a new constitutive structure from a refit of old variables~\cite{Cranford2010Materiomics:Macro,Lee2022AnDiscovery,Fish2021MesoscopicMaterials,Buehler2016IntegratedPerspective,Jackson2019RecentDesign,Anand2020ContinuumSolids,Shi2022HydrolysisAcid,Tang2009Constitutive}) 
become design principles for agentic AI systems that revise their own scientific vocabulary.

\hlyellow{The resulting contribution is a single audit framework for scientific discovery systems. Scientific state is modeled as typed artifacts over a schema: a copresheaf $I_t:\Sschema_b\to\Set$ at the committed layer, with realized provenance given by $\int_{\Sschema_b}I_t$ or by a multicategorical provenance object when artifacts have multiple parents. Search refines this population inside a fixed schema, whereas discovery is a verified regime transition with preserved old evidence, Kan-extension transport, and residual content not explained by transport. MDL/AIC-style gates, stress tests, supersessions, and rejected alternatives are therefore treated as first-class typed audit objects rather than external annotations. Builder/Breaker supplies a quantitative protein-mechanics trace in which a symbolic law is revised under an MDL gate \citep{buehler_break_world_2026}. CategoryScienceClaw implements the same logic as proof-carrying typed provenance over ScienceClaw, including deterministic materials inputs, accepted and rejected screening models, gate records, generated figures, reports, and claim artifacts.}
\hlyellow{The CategoryScienceClaw application is a high-entropy-alloy (HEA) fatigue-crack-growth investigation with direct materials-science content. A threshold-only $\Delta K_{th}$ ranking is rejected because FCC+BCC records in the open HEA fatigue database have higher median threshold than FCC records but also substantially steeper Paris slopes. The accepted representation is a two-axis damage-tolerance map that separates crack-growth threshold from post-threshold crack-growth stability and flags high-threshold/high-slope dual-phase records as a prospective validation class. A fiber-network model-comparison audit is included in the Supplementary Information as a supporting provenance example.}
\hlyellow{Formal definitions are gathered in the Materials and Methods section} (Section~\ref{sec:methods}); \hlyellow{the Results section develops their practical interpretation through these case studies.}

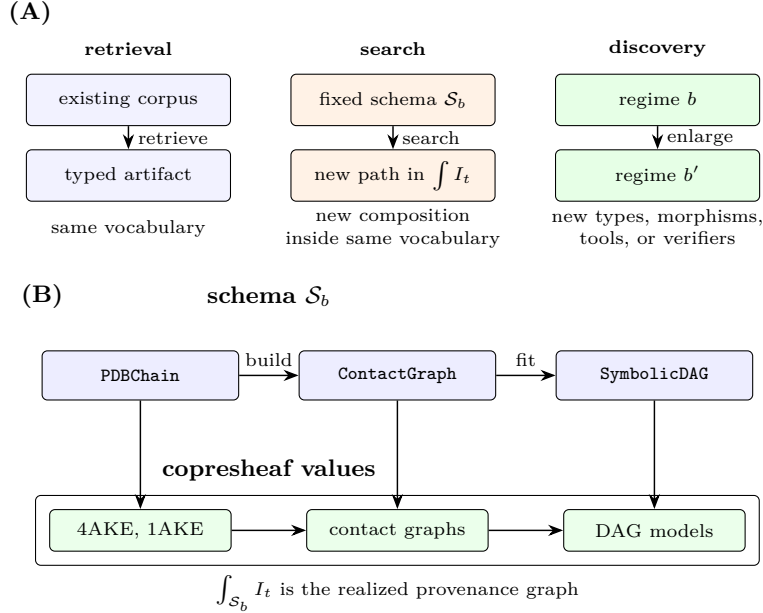
\begin{figure}[ht]
\centering
\begin{tikzpicture}[
  box/.style={draw,rounded corners=2pt,fill=blue!5,minimum width=2.7cm,minimum height=0.68cm,align=center,font=\scriptsize},
  sbox/.style={draw,rounded corners=2pt,fill=orange!10,minimum width=2.7cm,minimum height=0.68cm,align=center,font=\scriptsize},
  dbox/.style={draw,rounded corners=2pt,fill=green!10,minimum width=2.7cm,minimum height=0.68cm,align=center,font=\scriptsize},
  arr/.style={-{Stealth[length=2mm]},line width=0.55pt}
]
\node[font=\bfseries\small] at (-1.3,1.15) {(A)};
\node[box] (r1) at (0,0) {existing corpus};
\node[box] (r2) at (0,-1.0) {typed artifact};
\draw[arr] (r1) -- node[right,font=\scriptsize]{retrieve} (r2);
\node[font=\bfseries\scriptsize] at (0,0.68) {retrieval};

\node[sbox] (s1) at (3.5,0) {fixed schema $\Sschema_b$};
\node[sbox] (s2) at (3.5,-1.0) {new path in $\int I_t$};
\draw[arr] (s1) -- node[right,font=\scriptsize]{search} (s2);
\node[font=\bfseries\scriptsize] at (3.5,0.68) {search};

\node[dbox] (d1) at (7.0,0) {regime $b$};
\node[dbox] (d2) at (7.0,-1.0) {regime $b'$};
\draw[arr] (d1) -- node[right,font=\scriptsize]{enlarge} (d2);
\node[font=\bfseries\scriptsize] at (7.0,0.68) {discovery};
\node[font=\scriptsize,align=center] at (0,-1.7) {same vocabulary};
\node[font=\scriptsize,align=center] at (3.5,-1.7) {new composition\\inside same vocabulary};
\node[font=\scriptsize,align=center] at (7.0,-1.7) {new types, morphisms,\\tools, or verifiers};
\end{tikzpicture}

\vspace{0.6em}

\begin{tikzpicture}[
  type/.style={draw,rounded corners=2pt,fill=blue!7,minimum width=2.6cm,minimum height=0.65cm,align=center,font=\scriptsize},
  art/.style={draw,rounded corners=2pt,fill=green!8,minimum width=2.4cm,minimum height=0.55cm,align=center,font=\scriptsize},
  arr/.style={-{Stealth[length=2mm]},line width=0.55pt}
]
\node[font=\bfseries\small] at (-1.3,1.05) {(B)};
\node[font=\bfseries\small] at (1.7,1.05) {schema $\Sschema_b$};
\node[type] (pdb) at (0,0) {\texttt{PDBChain}};
\node[type] (cg) at (3.4,0) {\texttt{ContactGraph}};
\node[type] (mdl) at (6.8,0) {\texttt{SymbolicDAG}};
\draw[arr] (pdb) -- node[above,font=\scriptsize]{build} (cg);
\draw[arr] (cg) -- node[above,font=\scriptsize]{fit} (mdl);

\node[font=\bfseries\small] at (1.7,-1.25) {copresheaf values};
\node[art] (p1) at (0,-2.05) {4AKE, 1AKE};
\node[art] (c1) at (3.4,-2.05) {contact graphs};
\node[art] (m1) at (6.8,-2.05) {DAG models};
\draw[arr] (pdb) -- (p1);
\draw[arr] (cg) -- (c1);
\draw[arr] (mdl) -- (m1);
\draw[arr] (p1) -- (c1);
\draw[arr] (c1) -- (m1);

\node[draw,rounded corners=2pt,fit=(p1)(m1),inner sep=0.18cm,label={[font=\scriptsize]below:$\int_{\Sschema_b} I_t$ is the realized provenance graph}] {};
\end{tikzpicture}
\caption{(A) Retrieval, search, and discovery are structurally different operations. Retrieval adds an already representable artifact. Search finds a new path or object inside a fixed schema. Discovery changes the regime in which artifacts and operations are typed. (B) A fixed regime has a schema category $\Sschema_b$ of types and operations. A copresheaf $I_t:\Sschema_b\to\Set$ assigns actual artifacts to each type. The category of elements $\int I_t$ is the realized typed artifact DAG. \hlyellow{The labels 4AKE and 1AKE denote adenylate kinase structures by PDB identifier.}}
\label{fig:retrieval_search_discovery}
\label{fig:copresheaf_elements}
\end{figure}

% =====================================================================
\section{Results and Discussion}
\label{sec:results}
% =====================================================================

\subsection{Agentic discovery systems are typed artifact systems}
\label{sec:typed_artifact_systems}

An agentic discovery system is best understood as a typed artifact system. Its persistent state is not a conversation transcript, a hidden vector, or a single model checkpoint. It is a growing record of artifacts and their provenance: data, simulations, models, hypotheses, code, measurements, reports, critiques, and decisions. Each artifact has a type, and each operation has a declared source and target type. Critically this typing is not bureaucratic overhead but what distinguishes a scientific claim from a fluent answer, addressing an important limitation of many probabilistic AI systems.

Five components recur across the systems considered here.
\begin{enumerate}[leftmargin=1.7em]
    \item \textbf{A schema of artifact types and operations.} Types include sequences, structures, contact graphs, trajectories, hypotheses, symbolic models, measurements, and reports. Operations include structure prediction, simulation, normal-mode extraction, retrieval, proof checking, dimensional analysis, fabrication, and scoring.
    \item \textbf{A population of artifacts over that schema.} The system stores actual sequences, structures, simulations, equations, figures, claims, and reports inhabiting the declared types.
    \item \textbf{A provenance graph.} Every accepted artifact records its parents and the operation that produced it. Composition of operations is the scientific lineage from raw input to claim.
    \item \textbf{A gate or verifier.} New artifacts are not automatically committed. They are accepted, rejected, superseded, or held for review by an explicit gate. MDL is one such gate; pressure scoring, schema overlap, peer review, and community feedback are others.
    \item \textbf{A regime-update mechanism.} When evidence cannot be represented in the current schema, the system must extend or revise the schema, grammar, verifier, or tool registry.
\end{enumerate}

The distinction among retrieval, search, and discovery is shown schematically in Fig.~\ref{fig:retrieval_search_discovery}A. Retrieval adds artifacts already expressible in the schema. Search explores combinations of existing artifacts and operations. Discovery changes the regime by adding or revising the types and operations under which future artifacts are judged.

This view accommodates both compact experimental systems and large distributed systems. In ProtAgents, the schema contains sequences, structures, force predictions, and protein-design hypotheses \citep{ghafarollahi_protagents_2024}. In Sparks and SciAgents, it contains research goals, generated hypotheses, executable plans, code outputs, and reports \citep{ghafarollahi_sparks_2025,ghafarollahi_sciagents_2025}. ScienceClaw $\times$ Infinite instantiates the same mathematics at a larger architectural scale \citep{wang_scienceclaw_infinite_2026}\footnote{Code at \url{https://github.com/lamm-mit/scienceclaw}, \url{https://github.com/lamm-mit/infinite}, and the CategoryScienceClaw mechanics branch at \url{https://github.com/lamm-mit/scienceclaw/tree/categoryscienceclaw-mechanics}.}. ScienceClaw supplies the execution substrate: an extensible registry of typed scientific skills, immutable artifacts with metadata and parent lineage, shared open needs, plannerless coordination, pressure scoring, and mutation of the active artifact graph. Infinite supplies the discourse substrate: structured posts, hypotheses, methods, findings, links among claims, votes, comments, reputation, and moderation. Together they make scientific work auditable across both computation and communication. In CategoryScienceClaw, that ScienceClaw substrate is lifted into typed categorical state: the schema contains typed skills, immutable artifacts with parent lineage, pressure-ranked open needs, workflow mutation records, public discourse objects, domain inputs, descriptors, candidate models, accepted and rejected alternatives, gates, stress tests, regime-transition records, figures, and reports. The objects differ, but the structural skeleton is the same.

\subsection{Artifact states are copresheaves}
\label{sec:copresheaf_results}

The clean mathematical object behind the typed artifact graph is a copresheaf (Definition~\ref{def:copresheaf}). Fix a discovery regime $b$ (Definition~\ref{def:regime}). The regime includes a schema category $\Sschema_b$: its objects are artifact types and its morphisms are allowed operations. A scientific state at time $t$ is a covariant Set-valued functor
\begin{equation}
I_t:\Sschema_b\longrightarrow \Set .
\label{eq:copresheaf_state}
\end{equation}
For each type $A$, the set $I_t(A)$ contains the artifacts of type $A$ currently available to the system. For each operation $f:A\to B$, the function $I_t(f):I_t(A)\to I_t(B)$ records how an artifact of type $A$ gives rise to an artifact of type $B$, when that operation has been realized. 

For a materials scientist, typical fibers are:
\[
I_t(\mathrm{PDBChain}),\qquad
I_t(\mathrm{ContactGraph}),\qquad
I_t(\mathrm{SymbolicDAG}),
\]
the current PDB chains, the contact graphs constructed from them, and the candidate or accepted symbolic models. The actual provenance DAG is recovered by the category of elements $\int_{\Sschema_b} I_t$ (Fig.~\ref{fig:copresheaf_elements}B). Its objects are pairs $(A,x)$, where $A$ is a type and $x\in I_t(A)$ is an artifact of that type. A morphism $(A,x)\to(B,y)$ is an operation $f:A\to B$ such that $I_t(f)(x)=y$. Thus the category of elements is not a metaphor for provenance. It is the typed provenance graph.

%\Lan_u(\Phi_b)
%\Lan_u(\delta_t)

Copresheaves are the right level of abstraction because they separate a regime from its current contents. The schema can remain fixed while the artifact population grows. Conversely, a discovery move can change the schema while preserving the old artifact population by transport into a new regime. This separation is exactly what is missing when all artifacts, operations, and updates are collapsed into one untyped graph.

This construction can be understood as a typed generalization of a scientific knowledge graph.
A conventional knowledge graph records entities and relations over a largely fixed ontology.
Here the fibers \(I_t(A)\) contain scientific artifacts; morphisms \(f:A\to B\) are executable
or audited operations; the category of elements records realized provenance; gates record
commitment and rejection; and regime transitions allow the vocabulary itself to change.

For a formal definition, fix a discovery regime
\[
b=(\Sschema_b,\Gamma_b,V_b,L_b).
\]
A categorical knowledge--computation graph at time \(t\) is the tuple
\[
\mathfrak K_t^b =
\bigl(
\Sschema_b,\Gamma_b,I_t,\mathsf{Prov}_t,V_b,L_b,\mathsf D_t,\pi_t
\bigr),
\]
where \(I_t:\Sschema_b\to\Set\) is the typed artifact state;
\(\mathsf{Prov}_t\) is the realized provenance object, equal to
\(\int_{\Sschema_b} I_t\) in the unary case and to the corresponding typed
multicategory or colored-operadic provenance object when artifacts have multiple
parents; \(V_b\) is the verifier or gate; \(L_b\) is the description-length or
model-selection functional when present; \(\mathsf D_t\) is an optional discourse
category of claims, posts, evidence bundles, objections, replications, and validation
signals; and \(\pi_t\) is the publication map carrying audited artifact paths or
hyperpaths into public claim objects when such a discourse layer is implemented.

A conventional scientific knowledge graph is recovered by fixing
\(\Sschema_b\), forgetting \(\Gamma_b,V_b,L_b,\mathsf D_t,\pi_t\), and viewing
\(\mathsf{Prov}_t\) only as an entity-relation graph. A workflow-provenance graph
is recovered by retaining production edges but not the full gate, rejection, discourse,
and regime-transition structure. Hence \(\mathfrak K_t^b\) is not simply a knowledge
graph with metadata but instead an executable, verifier-aware, provenance-preserving
scientific state. Fixed-regime search is evolution of \(\mathfrak K_t^b\) under
\(\Phi_b\); discovery is a verified transition
\[
\mathfrak K_t^b \longrightarrow \mathfrak K_{t+1}^{b'}
\]
with transported evidence \(\Lan_u I_t\) and residual content outside
\(\operatorname{im}(\bar\rho)\).

\begin{remark}[Yoneda reading]
\label{rem:yoneda}
For a type $A$, the representable copresheaf $\Sschema_b(A,-)$ encodes all operations that can be applied to a generic artifact of type $A$. An actual artifact $x\in I_t(A)$ induces, by the covariant Yoneda lemma, a natural transformation $\Sschema_b(A,-)\to I_t$: each operation out of $A$ is sent to the artifact obtained by applying that operation to $x$. In implementation terms, an artifact is known by the typed operations it can participate in and the downstream artifacts those operations produce.
\end{remark}

\subsection{Fixed-regime updates are endofunctorial under explicit assumptions}
\label{sec:endofunctor_results}

Inside a fixed regime $b$, an agentic system updates artifact states. Write $[\Sschema_b,\Set]$ for the functor category of copresheaves on $\Sschema_b$. Abstractly, a fixed-regime update is an operation on copresheaves,
\begin{equation}
\Phi_b:[\Sschema_b,\Set]\longrightarrow [\Sschema_b,\Set].
\label{eq:fixed_update}
\end{equation}
It reads the current artifact population, selects compatible operations, proposes new artifacts, applies a gate, and returns the next artifact population. The statement that $\Phi_b$ is an endofunctor is a further preservation claim: after the category of admissible refinements between artifact states is specified, $\Phi_b$ must extend from an object-level update to a map on refinement morphisms in the corresponding category of knowledge states $\K_b$ (Definition~\ref{def:fixedupdate}). For the structural observations and the Kan proposition, $\K_b$ is the subcategory of $[\Sschema_b,\Set]$ in which morphisms are componentwise injective natural transformations (Methods, Section~\ref{sec:methods}); refinements may add, annotate, or supersede artifacts, but may not silently identify two previously distinct accepted artifacts.

This qualification matters. A raw program that takes a JSON artifact ledger to another JSON artifact ledger is only an endomap. It becomes an endofunctor only when it also preserves refinement morphisms: if one artifact state extends another by adding verified artifacts without overwriting prior provenance, the updated state should extend the updated predecessor in the same way. This is the formal version of a familiar engineering requirement: if a pipeline is refactored, old valid workflows must still compose.

In code, this is an audit contract rather than a slogan. The implementation must maintain stable artifact identifiers, typed tool or skill signatures, explicit parent lineage, append-only or explicit supersession semantics, status records for failed or retried calls, and no silent merge or deletion of accepted artifacts. When these conditions hold at the committed-state layer, a refinement $\alpha:I\to J$ can be pushed through the update to give a refinement $\Phi_b(\alpha):\Phi_b(I)\to\Phi_b(J)$. When they fail, the system may still be useful, but the endofunctor model applies only after adding the missing audit structure.

\hlyellow{The update may be implemented by agents, simulators, retrieval systems, symbolic search, numerical solvers, or human feedback. The categorical requirement is not that the implementation be deterministic, exact, or neural-free, but that the committed state preserve typed provenance. Numerical and statistical artifacts are interpreted relative to the verifier or gate of the regime, which declares the relevant tolerance, equivalence, or model-selection criterion. In stochastic, asynchronous, partial, or approximate implementations, the same condition is imposed on the committed-state relation, stochastic kernel, partial map, or Kleisli morphism rather than on every raw execution trace \citep{moggi_monads_1991}. Multi-parent scientific operations are represented by the multicategorical variants discussed below.}

\hlyellow{The Set-valued copresheaf should therefore be read as the clean committed-state semantics after gates have accepted, rejected, or superseded proposed artifacts. The present implementation checks typed provenance, accepted/rejected gate records, deterministic mechanics computations, replay files, and post-hoc categorical audits; it does not claim total Set-valued functorial certification of each low-level proposal, retry, failure, or model call.}

This distinction is important in the examples below. Builder/Breaker is closest to the strict model at the layer of MDL-accepted symbolic DAGs and accumulated evidence, not at the raw proposal trace. CategoryScienceClaw is close at the mechanics-audit layer because typed skills, immutable artifacts, parent lineage, candidate models, accepted and rejected alternatives, gates, stress tests, regime transitions, reports, and public discourse are recorded explicitly; a fully endofunctorial implementation would further require checked schema signatures and explicit refinement maps between artifact states.

\subsection{Discovery is regime transition, not only iteration}

\label{sec:regime_enlargement_results}

Search is iteration of $\Phi_b$ inside a fixed regime. Discovery requires a transition to a new regime (Definition~\ref{def:transition}). Let
\begin{equation}
u:\Sschema_b\longrightarrow \Sschema_{b'}
\label{eq:schema_extension}
\end{equation}
be a schema map from the old regime to the new one. In the simplest case, $u$ is an inclusion that preserves old types and operations while adding new ones. In more realistic cases, it may also refine old types, split a type into subtypes, add a verifier, or add new morphisms between old objects. Therefore the general notion should not be restricted to a fully faithful embedding of categories; discovery may add new admissible relations among old objects.

The transition is summarized in Fig.~\ref{fig:kan_extension}A. The old artifact state is transported into the new schema by left Kan extension:
\begin{equation}
\Lan_u I_t:\Sschema_{b'}\longrightarrow \Set .
\label{eq:lan_transport}
\end{equation}
For an object $A'$ of the new schema, the value $(\Lan_u I_t)(A')$ is computed as a colimit over the old artifacts that map toward $A'$ (the precise formula appears as Eq.~\ref{eq:lan_formula} in Methods). Operationally, it is the least systematic way to reinterpret old artifacts inside the new vocabulary. If $A'$ receives no morphisms from the image of $u$, the comma category indexing the colimit is empty, so $(\Lan_u I_t)(A')=\varnothing$: free transport supplies nothing at that isolated new type. If $A'$ does receive a morphism from an old type, then old evidence can be transported to it, even if $A'$ itself is a new object. A discovery move is therefore not merely $\Lan_u I_t$; it is transport plus a verified post-transition state containing new evidence, new artifacts, new verifier outcomes, or new grammar productions that are not accounted for by transport alone.

This gives a precise form to the slogan that discovery changes the world model. Yet importantly it does not mean the old world disappears. The old artifacts persist as transported evidence, and their old provenance must remain auditable. What changes is the regime in which the evidence can be represented and composed.

The Kan-extension picture also gives a quantitative reading of how much discovery has occurred, but only after one additional datum is specified. A verified transition includes a natural transformation
\[
\rho:I_t\longrightarrow u^* I'_{t+1},
\]
where $u^*$ restricts a new-regime state back to the old schema. This map says, explicitly, how each old accepted artifact is preserved in the new state. By the adjunction $\Lan_u\dashv u^*$, $\rho$ corresponds uniquely to a comparison map
\[
\bar\rho:\Lan_u I_t\longrightarrow I'_{t+1}.
\]
The image of $\bar\rho$ is the transported-evidence substate. Artifacts outside this image, object by object, are the empirical or representational content the system had to acquire beyond functorial reinterpretation of old evidence. When the new regime carries a relative description-length functional $L_{b'}(-\mid-)$, the \textbf{discovery cost} is the bit budget $L_{b'}(I'_{t+1}\mid\mathrm{im}(\bar\rho))$ required to specify the post-transition state given transported evidence. The Builder/Breaker model in Section~\ref{sec:builder_breaker_results} uses the MDL gate to measure this kind of cost in a concrete protein-mechanics regime.

\subsection{The Builder/Breaker model gives a quantitative MDL case}
\label{sec:builder_breaker_results}

The Builder/Breaker protein-mechanics model is a compact empirical instance of the framework, building on preliminary work reported in \citep{buehler_break_world_2026} (Definition~\ref{def:bb}\footnote{Code at \url{https://github.com/lamm-mit/BreakingTheWorld}.}). The Breaker chooses new proteins intended to expose failure modes of the current symbolic model. The Builder proposes symbolic DAG edits. The gate accepts a candidate only if it reduces total description length after paired refitting on the same accumulated evidence:
\begin{equation}
L(M,D)=L_{\mathrm{model}}(M)+L_{\mathrm{data}}(D\mid M).
\label{eq:case_mdl_results}
\end{equation}
Here $M$ is the symbolic DAG world model and $D$ is the accumulated protein-mechanics evidence. If the Breaker adds stress-test evidence $E$, a proposed revision $M'$ becomes part of the world model only when
\[
L(M',D\cup E)<L(M,D\cup E),
\]
after both models have been judged on the same evidence. This is the formal sense in which a productive failure becomes scientific structure: the new symbolic law must explain the counterexamples well enough to pay for its additional bits.
Consequently, the acceptance criterion in this case is not monotone improvement of a single
predictive score across iterations. Each outer iteration changes the accumulated evidence set
on which the current and proposed symbolic DAGs are compared. The relevant paired test is
therefore whether the revised model compresses the enlarged evidence set better than the
previous model after both are refit on that same evidence, not whether the reported $R^2$
forms a monotone sequence across heterogeneous stages.

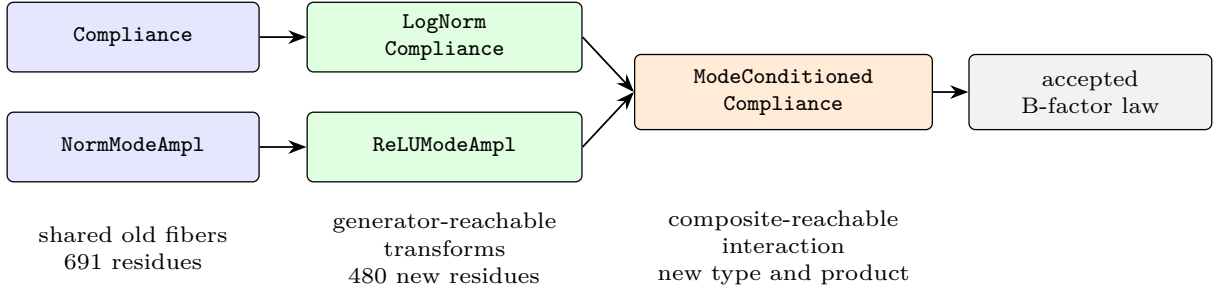
\begin{figure}[t]
\centering
{\small(A)}\quad
\begin{tikzcd}[row sep=large,column sep=large]
I_t \arrow[r, "\delta_t"] \arrow[d, "\Lan_u"']& I_{t+1} \arrow[d, "\Lan_u"] \\
\Lan_u I_t \arrow[r, dashed, "\Lan_u(\delta_t)"'] & \Lan_u I_{t+1} \arrow[r, "\bar\rho_{t+1}"] & I'_{t+1}
\end{tikzcd}

\vspace{0.8em}

{\small(B)}\quad
\resizebox{0.92\linewidth}{!}{%
\begin{tikzpicture}[
  font=\scriptsize,
  panel/.style={anchor=west,font=\bfseries\sffamily\scriptsize},
  shared/.style={draw,rounded corners=2pt,fill=blue!10,minimum width=2.75cm,minimum height=0.76cm,align=center,font=\scriptsize},
  unary/.style={draw,rounded corners=2pt,fill=green!12,minimum width=3.0cm,minimum height=0.76cm,align=center,font=\scriptsize},
  comp/.style={draw,rounded corners=2pt,fill=orange!16,minimum width=3.25cm,minimum height=0.82cm,align=center,font=\scriptsize},
  target/.style={draw,rounded corners=2pt,fill=gray!10,minimum width=2.65cm,minimum height=0.82cm,align=center,font=\scriptsize},
  note/.style={font=\scriptsize,align=center,text width=3.2cm},
  arr/.style={-{Stealth[length=2mm]},line width=0.55pt}
]

\node[panel] at (-6.6,0.55) {Final transition $S_2\to S_3$: mode-conditioned compliance};

\node[shared] (c) at (-5.2,-0.35) {\texttt{Compliance}};
\node[shared] (m) at (-5.2,-1.55) {\texttt{NormModeAmpl}};
\node[unary] (lc) at (-1.8,-0.35) {\texttt{LogNorm}\\\texttt{Compliance}};
\node[unary] (rm) at (-1.8,-1.55) {\texttt{ReLUModeAmpl}};
\node[comp] (mc) at (1.9,-0.95) {\texttt{ModeConditioned}\\\texttt{Compliance}};
\node[target] (bf) at (5.25,-0.95) {accepted\\B-factor law};

\draw[arr] (c) -- (lc);
\draw[arr] (m) -- (rm);
\draw[arr] (lc.east) -- (mc.west);
\draw[arr] (rm.east) -- (mc.west);
\draw[arr] (mc) -- (bf);

\node[note] at (-5.2,-2.65) {shared old fibers\\$691$ residues};
\node[note] at (-1.8,-2.65) {generator-reachable transforms\\$480$ new residues};
\node[note,text width=3.6cm] at (1.9,-2.65) {composite-reachable interaction\\new type and product};

\end{tikzpicture}%
}
\caption{(A) Fixed-regime operation is the update $\Phi_b$ inside a schema $\Sschema_b$. A committed fixed-regime step is represented by a refinement $\delta_t:I_t\to I_{t+1}$ associated with the object-level update $I_{t+1}=\Phi_b(I_t)$. The lower dashed arrow is $\mathrm{Lan}_u(\delta_t):\mathrm{Lan}_u I_t\to \mathrm{Lan}_u I_{t+1}$, not an independent new-regime dynamics; the left square commutes by functoriality of $\mathrm{Lan}_u$ on refinement morphisms. Discovery enters through the comparison map $\bar\rho_{t+1}$ from transported evidence to the verified post-transition state (Definition~\ref{def:transition}); in general $\Lan_u I_{t+1}\neq I'_{t+1}$, and the comparison map together with the residual artifacts outside its image record what the discovery move added beyond functorial transport. (B) Kan-transport audit of the Builder/Breaker protein-mechanics run, instantiating this general picture for the final transition $S_2\to S_3$: log-compliance and shifted ReLU mode participation are generator-reachable transformations of old physics-derived quantities, whereas \texttt{ModeConditionedCompliance} is reachable only after the new regime admits a product operation. This is the categorical form of the mechanics insight in Eq.~\ref{eq:bfactor_spectral_form}.}
\label{fig:kan_extension}
\label{fig:kan_audit_builder_breaker}
\end{figure}

The relevant fixed schema contains PDB chains, C$\alpha$ coordinates, contact graphs, Kirchhoff matrices, \hlyellow{Gaussian Network Model (GNM) spectra}, compliance observables, slow-mode observables, feature terms, symbolic DAGs, B-factor targets, and MDL budgets. The physical base is the elastic-network and Gaussian Network Model tradition, where low-resolution harmonic contact networks explain slow collective motions and residue-level temperature factors from structure alone \citep{tirion_elastic_1996,bahar_gnm_1997,haliloglu_gaussian_1997}. For a protein chain $p$ with residues $i=1,\ldots,N_p$ and C$\alpha$ coordinates $\mathbf r_{pi}\in\mathbb R^3$, define the contact graph by
\[
A^{(p)}_{ij}
=\mathbf 1\{i\neq j,\ \|\mathbf r_{pi}-\mathbf r_{pj}\|<r_c\},
\qquad r_c=10.0~\text{\AA}.
\]
The GNM Kirchhoff matrix is
\[
\Gamma^{(p)}_{ij}=
\begin{cases}
-A^{(p)}_{ij}, & i\neq j,\\
\sum_{k\neq i} A^{(p)}_{ik}, & i=j .
\end{cases}
\]
Diagonalize
\[
\Gamma_p\mathbf u_{pk}=\lambda_{pk}\mathbf u_{pk},
\qquad
0=\lambda_{p1}\leq\lambda_{p2}\leq\cdots\leq\lambda_{pN_p}.
\]
The all-mode compliance of residue $i$ is the diagonal of the pseudoinverse,
\[
C_{pi}
=(\Gamma_p^+)_{ii}
=\sum_{\lambda_{pk}>0}\frac{u_{pik}^2}{\lambda_{pk}}.
\]
For connected contact graphs this is the familiar sum over $k=2,\ldots,N_p$; the positive-eigenvalue form is the correct pseudoinverse expression if a cutoff or chain segmentation produces additional zero modes.
This is not an arbitrary learned feature but instead the harmonic-network compliance implied by the residue contact topology. In the standard GNM relation,
\[
\langle \Delta R_{pi}^2\rangle
=\frac{3k_B T}{\gamma}\,C_{pi},
\qquad
B_{pi}^{\mathrm{GNM}}
=\frac{8\pi^2}{3}\langle \Delta R_{pi}^2\rangle,
\]
where $\gamma$ is the effective spring constant. Because the learning target is normalized within each chain, the global constants $T$, $\gamma$, and $8\pi^2/3$ drop out. The experimental target is the per-chain normalized C$\alpha$ crystallographic B-factor,
\[
B^{(z)}_{pi}=\frac{B_{pi}-\overline{B}_p}{s_{B,p}},
\]
so the model explains within-chain flexibility patterns rather than absolute crystallographic scale. The physically important caveat is that crystallographic B-factors include thermal motion, static disorder, refinement effects, and crystal-packing effects; the discovery task is therefore to find a compact structural proxy for the experimental pattern, not a complete molecular dynamics law.

Let $z_p(x_i)=(x_i-\overline{x}_p)/s_{x,p}$ denote per-chain normalization. The log-compliance feature and slowest collective-mode participation are
\[
\phi_{pi}=z_p\!\left(\log(C_{pi}+\epsilon)\right),
\qquad
\psi_{pi}=\left[z_p\!\left(|u_{pi2}|\right)+\theta\right]_+ ,
\]
where $[x]_+=\max(x,0)$, $\epsilon>0$ is a small numerical floor inside the logarithm, and $u_{p2}$ denotes the first nonzero GNM eigenmode, called \texttt{mode1\_abs\_z} in the run after taking absolute value and z-scoring. If a contact graph has more than one zero mode, $u_{p2}$ should be read as the first positive-eigenvalue mode. The shift $\theta$ is added before the ReLU, so it acts as a lower clip rather than a conventional threshold; the equivalent threshold form is $[z_p(|u_{pi2}|)-\tau]_+$ with $\tau=-\theta$, and the fitted $\theta=2.2678$ corresponds to clipping near the lowest observed value of $z_p(|u_{pi2}|)$ on the mixed validation slice, so $\psi_{pi}$ shifts first-mode participation onto a positive scale and zeroes only the residues that barely participate in the dominant collective deformation. The discovered symbolic law has the form
\begin{equation}
\boxed{
\widehat B^{(z)}_{pi}
=\alpha+\beta\,\phi_{pi}\psi_{pi}
}
\label{eq:bfactor_symbolic_form}
\end{equation}
or, expanded in the GNM spectrum,
\begin{equation}
\boxed{
\begin{aligned}
\widehat B^{(z)}_{pi}
&=\alpha+\beta\,
z_p\!\left(
\log\!\left(\sum_{\lambda_{pk}>0}\frac{u_{pik}^2}{\lambda_{pk}}+\epsilon\right)
\right)\\
&\quad{}\times
\left[z_p\!\left(|u_{pi2}|\right)+\theta\right]_+ .
\end{aligned}
}
\label{eq:bfactor_spectral_form}
\end{equation}
The fitted numerical values accepted by the MDL gate are $\alpha=-0.1332$, 
$\beta=0.2239$, and $\theta=2.2678$.
The important discovery event is not merely the use of a normal mode; normal modes are available from the start. Categorically, this is not the appearance of an isolated new object. It is the transition at which the schema admits a new multi-input morphism
\[
\texttt{LogNormCompliance}\times \texttt{ReLUModeAmpl}
\longrightarrow
\texttt{ModeConditionedCompliance},
\]
whose target is composite-reachable from old physics-derived quantities only after the new unary transforms and product operation are admitted. The audit below makes this composite reachability explicit. Mechanically, the resulting interaction supports a new explanatory role: local softness matters most when it is expressed along a global collective deformation, rather than acting as an independent additive cause.

For additional context the four outer iterations are summarized in Fig.~\ref{fig:case_dagmdl}. Iteration 0 fits a minimal local fluctuation model on compact proteins. Iteration 1 adds boundary and slow-mode structure for terminal flexibility and gains $9.0$ bits. Iteration 2 exposes a hinge/domain-motion regime, using open and closed adenylate kinase (PDB chains 4AKE and 1AKE) as the canonical conformational stress test, and reorganizes the DAG around collective-motion interpretation at $37.3$ bits gain. Iteration 3 searches inside the enlarged regime on a mixed validation slice and consolidates the model into the compact multiplicative law in Eq.~\ref{eq:bfactor_spectral_form}, at $54.3$ bits gain. The point of the trajectory is not that successive regressions improved monotonically; it is that the admissible symbolic vocabulary was repeatedly stress-tested, revised, and compressed until the surviving law expressed a new mechanics relation.

Equation~\ref{eq:bfactor_spectral_form}, with the numerical values above, is the final accepted symbolic world model for the run. It has a clear mechanical reading. The factor $\phi_{pi}=z_p(\log(C_{pi}+\epsilon))$ is a compressed local-compliance coordinate: it is positive for residues whose contact-network environment predicts above-average fluctuation and negative for mechanically buried or strongly constrained residues. The factor $\psi_{pi}=[z_p(|u_{pi2}|)+2.2678]_+$ is a nonnegative participation weight for the slowest collective mode; the threshold is near the lowest observed value of $z_p(|u_{pi2}|)$ in the mixed validation stage, so the ReLU mostly shifts first-mode participation onto a positive scale and suppresses residues that barely participate in the dominant collective deformation. Their product says that experimental B-factor variation is best compressed by \emph{mode-conditioned compliance}: a residue has high predicted B-factor when it is locally soft in the GNM sense and participates strongly in the dominant collective motion. Local softness that is not aligned with the slow mode is down-weighted, and slow-mode participation without local compliance is insufficient.

The result is not a first-principles derivation from an atomistic Hamiltonian or a crystallographic refinement model, and it is not a neural-network predictor or unconstrained curve fit. The physics sits in the typed compositional pipeline
\[
\{\mathbf r_{pi}\}\longmapsto A_p\longmapsto \Gamma_p
\longmapsto \{(\lambda_{pk},\mathbf u_{pk})\}_{k}
\longmapsto (C_{pi},|u_{pi2}|)
\longmapsto \widehat B^{(z)}_{pi}.
\]
The discovery system searches over symbolic compositions of these physically meaningful artifacts and accepts a revised composition only when it compresses broader evidence better after paying for complexity. Thus the learned law is a mechanics-based constitutive surrogate: not ``B-factor equals an arbitrary fitted expression,'' but
\[
\text{within-chain B-factor pattern}
\sim
\text{all-mode elastic compliance}
\times
\text{slow collective-mode participation}.
\]
The log transform is also physically meaningful because raw GNM fluctuations are heavy-tailed; the model discovered that experimental B-factors are better described by a compressed fluctuation scale than by unbounded raw compliance. The key scientific claim is therefore structural: experimental protein flexibility is not governed by local elastic compliance alone, but by local compliance expressed through participation in the dominant collective mode of the contact-network spectrum.

\begin{table}[ht]
\centering
\footnotesize
\setlength{\tabcolsep}{3.5pt}
\begin{tabular}{>{\raggedright\arraybackslash}p{0.75cm}
                >{\raggedright\arraybackslash}p{1.4cm}
                >{\raggedright\arraybackslash}p{2.65cm}
                >{\raggedright\arraybackslash}p{2.45cm}
                >{\raggedright\arraybackslash}p{2.1cm}
                >{\raggedright\arraybackslash}p{1.0cm}
                >{\raggedright\arraybackslash}p{1.0cm}}
\toprule
Move & Break type & Generator-reachable new types & Composite-reachable new types & Retracted types & $\Delta L_M$ & MDL gain \\
\midrule
$0\to1$ & regime split & ReLU compliance; terminal exposure & boundary product & old linear parameters & $+39.1$ bits & $+9.0$ bits \\
$1\to2$ & ontology break & none beyond shared GNM base & none beyond shared GNM base & terminal and boundary feature family; old parameters & $-14.4$ bits & $+37.3$ bits \\
$2\to3$ & regime split & log-compliance; ReLU mode amplitude & mode-conditioned compliance & additive-model parameters & $-10.3$ bits & $+54.3$ bits \\
\bottomrule
\end{tabular}
\caption{Two-level Kan audit of the Builder/Breaker transitions.
Generator-reachable types receive an immediate unary morphism from old schema
objects. Composite-reachable types require new intermediate structure or a newly
admitted multi-input composition. Parameter objects are singletons in the audit;
they are listed only when they affect the interpretation of a transition. The
signed $\Delta L_{\rm model}$ records whether the symbolic model code grew or
shrank. MDL gains are acceptance gains from paired refitting on the same
accumulated evidence, not a direct numerical discovery cost. The retracted-types
column lists model commitments removed at each transition; retraction is recorded
as supersession, so the underlying evidence and its provenance remain available.
Read together with $\Delta L_{\rm model}$ and the MDL gain, the later accepted
transitions reduce the model code length (negative $\Delta L_{\rm model}$) while
still yielding the largest MDL gains, so the symbolic model becomes simpler even
as it compresses the broader evidence better.
}
\label{tab:kan_audit}
\end{table}

\begin{figure}[ht]
\centering
\includegraphics[width=\linewidth]{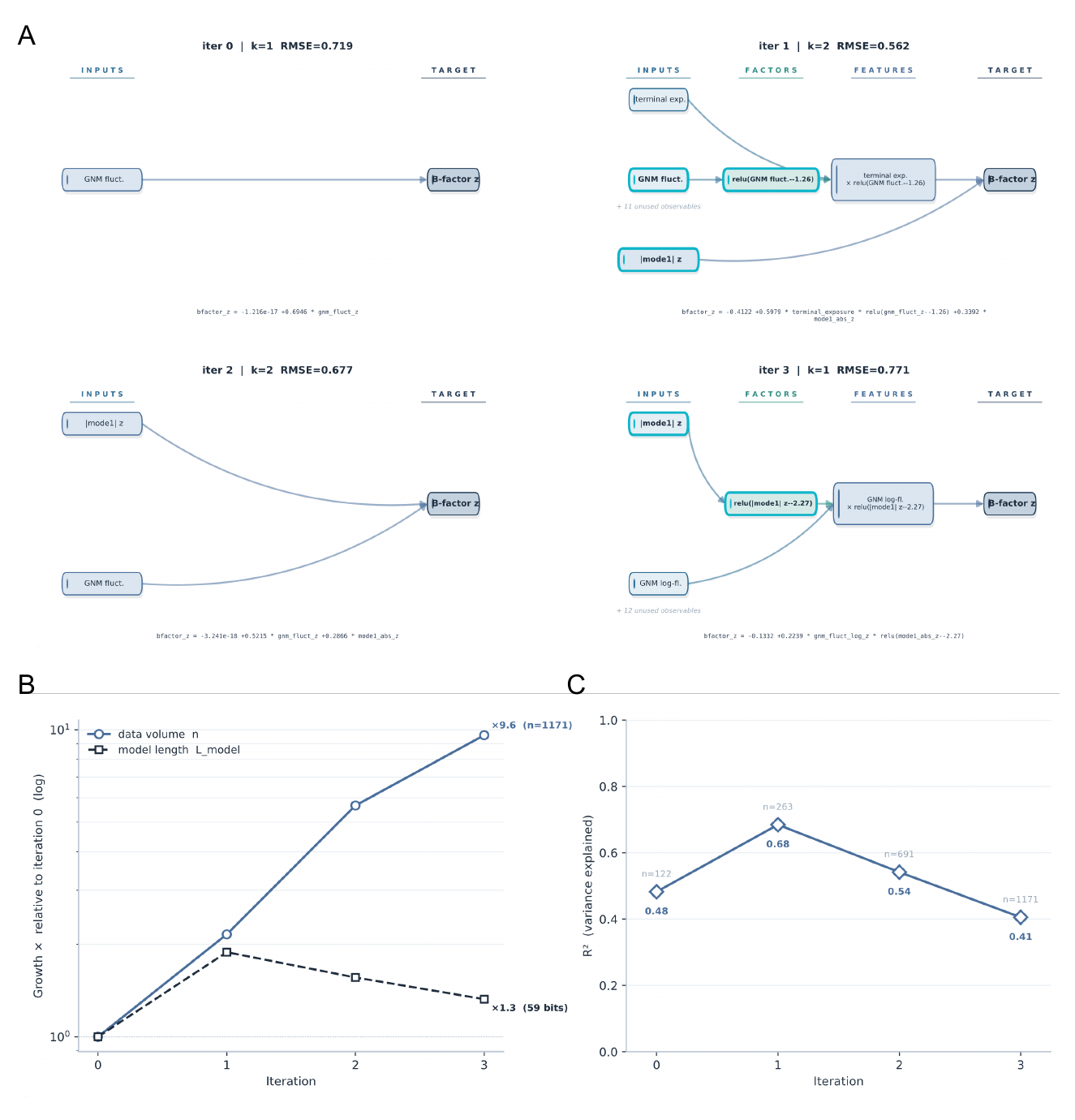}% 
\caption{
Parsimonious scaling of the discovered world model.
(A) Evolution of the world-model DAG across discovery iterations (0--3), read
left to right through four stages: inputs (observables), factors (nonlinear
transforms, e.g.\ thresholded ReLU terms), features (terms entering the linear
predictor), and the target ($B$-factor, $z$-scored). Nodes are colored by
stage, edges tinted by source, and nodes new to an iteration are outlined in
cyan; each panel lists the iteration, active-feature count $k$, and RMSE.
Complexity peaks at iteration~1 ($k=2$) and is pruned back to $k=1$ by
iteration~3 as newly revealed slices no longer justify the added term under MDL.
(B) Growth relative to iteration~0 (log axis): the data volume $n$ rises
$9.6\times$ ($122\rightarrow1171$ observations) while the model length
$L_{\mathrm{model}}$ rises only $1.3\times$ ($44\rightarrow59$~bits, $k\le2$)---an
order of magnitude more data absorbed without added complexity.
(C) Descriptive fit accuracy $R^2$ versus iteration, with the accumulated sample size $n$
annotated. These values should not be read as a monotone benchmark curve on a fixed test
set: each point is evaluated on the evidence available at that stage, and later stages include
harder stress-test proteins. The relevant trend is therefore joint parsimony under expanding
evidence: $n$ grows $\sim 10\times$ while $L_{\mathrm{model}}$ grows only $1.3\times$, and the
accepted symbolic law remains compact rather than being allowed to accumulate ad hoc
terms. Data based on earlier results \citep{buehler_break_world_2026}, with expanded analysis.
}
\label{fig:case_dagmdl}
\end{figure}

A post-hoc Kan-transport audit of the run makes the categorical content of this discovery explicit (Fig.~\ref{fig:kan_audit_builder_breaker}B and Table~\ref{tab:kan_audit}). The audit constructs a finite schema for each accepted outer iteration, transports the old accumulated artifact state into the next schema, and separates three cases. A new type is \emph{generator-reachable} when it receives an immediate unary generating morphism from an old type. A new type is \emph{composite-reachable} when it is not generator-reachable but becomes reachable through new intermediate types or multi-input operations admitted by the enlarged regime. A type is \emph{isolated} when it is unreachable even by such composites. This distinction matters for the final protein law. The features \texttt{LogNormCompliance} and \texttt{ReLUModeAmpl} are generator-reachable from old physics-derived quantities: they are new transformations of already available compliance and slow-mode amplitude. By contrast, \texttt{ModeConditionedCompliance} is composite-reachable only: it appears when the new regime admits the product operation
\[
\texttt{LogNormCompliance}\times \texttt{ReLUModeAmpl}
\longrightarrow
\texttt{ModeConditionedCompliance}.
\]
Thus the final discovery is not that GNM compliance or the slow collective mode exists; both are already present in the old schema. The new scientific commitment is the interaction type that lets local compliance be conditioned by participation in a collective deformation, and the MDL gate accepts this commitment with a $54.3$-bit gain on the accumulated evidence. The signed model-code changes reinforce the same interpretation: the first accepted transition increases model description length by $39.1$ bits, whereas the later accepted transitions reduce it by $14.4$ and $10.3$ bits. Discovery in this run is therefore not monotone accumulation of terms; it includes retraction and compression of the symbolic vocabulary.

The non-monotonic $R^2$ trajectory, $0.48 \to 0.68 \to 0.54 \to 0.41$, should therefore be
read as evidence-set widening rather than as a failed optimization trace. The four values are
not measurements of the same model class on a fixed benchmark. They are descriptive fits
on successively enlarged and more adversarial accumulated evidence sets, moving from
compact proteins to terminal-flexibility and hinge/domain-motion stress tests and finally to
the mixed validation slice. In such a setting, monotone $R^2$ would be the wrong success
criterion: it would reward adding terms to chase heterogeneous data. The MDL gate instead
asks whether a revised symbolic structure survives paired comparison on the same enlarged
evidence after paying its model-code cost. The final transition is therefore successful not
because it maximizes $R^2$, but because the multiplicative mode-conditioned-compliance
law remains the accepted compressed structure after additive and boundary terms fail to pay
for themselves.

\hlyellow{The inner symbolic search makes this selectivity visible (Fig.~}\ref{fig:case_hillclimb}\hlyellow{): in that iteration only 16 of 144 proposed DAG edits survive the MDL gate, and accepted moves include removals as well as additions. Tracking each feature slot across the run makes the churn explicit (Fig.~}\ref{fig:feature_lifecycle}\hlyellow{): the early iterations only accrete features, whereas iteration~3 explores four slots and retracts three, collapsing to the single surviving mode-conditioned law. Aggregated across all four iterations, the gate admits only 25 of 388 proposals (6.4\%), with operator-dependent selectivity: structure-recombining moves survive most often (seeds 21\%, swaps 11\%) while bare feature additions rarely do (3\%), and feature removals are a high-yield operator rather than rare cleanup (Fig.~}\ref{fig:gate_anatomy}\hlyellow{). The gate is therefore not decorative: it determines which proposed structures become committed scientific artifacts.}

\begin{figure}[h]
\centering
\includegraphics[width=.8\linewidth]{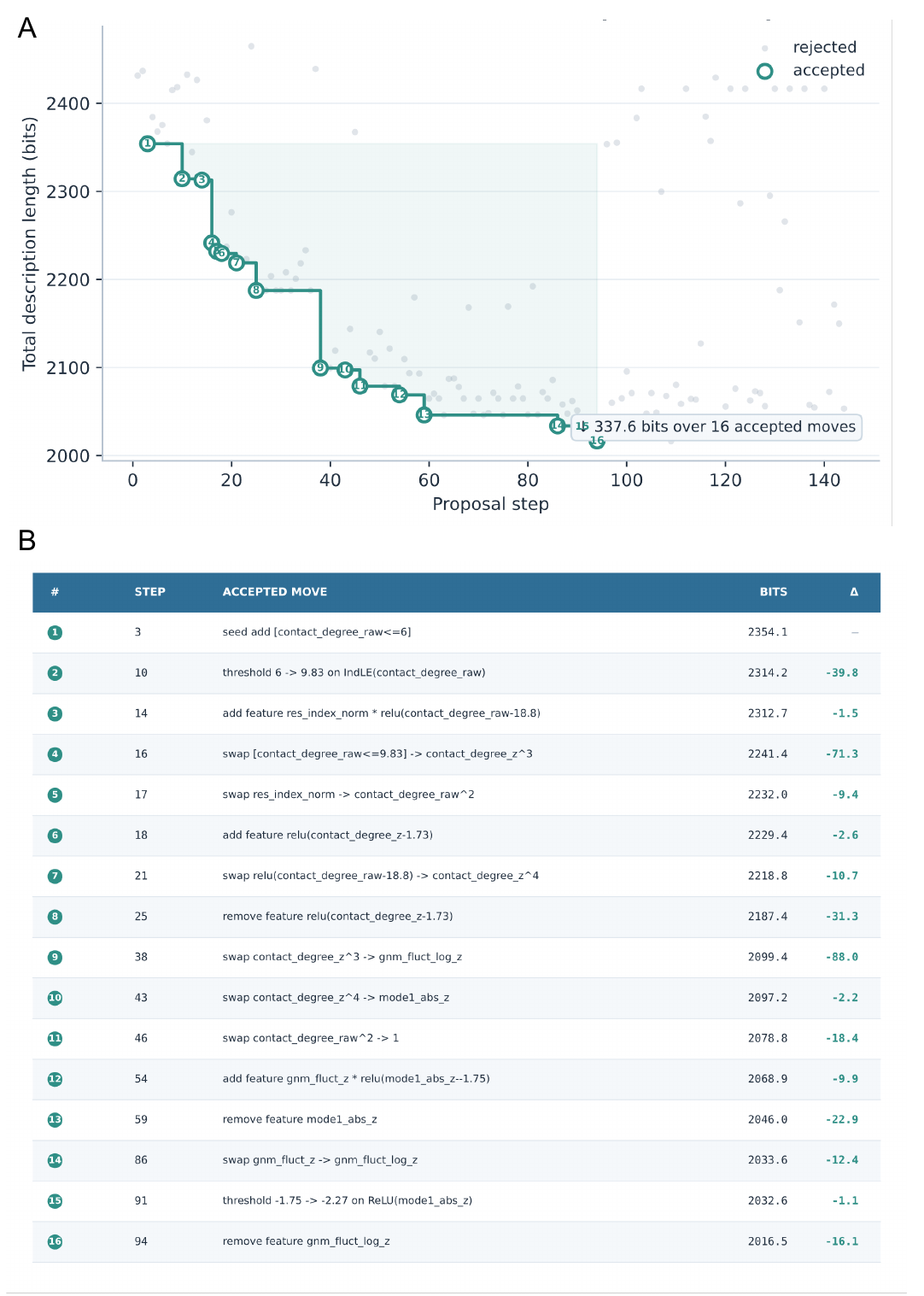}%
\caption{
Inner MDL-guided search within a discovery iteration (data from \citep{buehler_break_world_2026}).
(A) Hill-climb frontier: total description length (bits) versus proposal step.
Faint grey points are rejected proposals; the stepped teal curve is the
best-so-far frontier through the accepted moves (numbered markers), which
together reduce the description length by $337.6$ bits over $16$ accepted moves
(from $2354.1$ to $2016.5$ bits).
(B) Ledger of the accepted moves, keyed to the numbered markers in (A): each
row gives the proposal step, the structural edit (seed, add/remove a feature,
swap one feature for another, or adjust a threshold), the resulting total
description length in bits, and the change $\Delta$ relative to the previous
accepted state (negative $=$ improvement).
Together the panels show that the model is assembled through many small,
individually bit-reducing edits (adding, recombining, thresholding, and
pruning features) rather than a single large step.
}
\label{fig:case_hillclimb}
\end{figure}

\begin{figure}[ht]
\centering
\includegraphics[width=1.\linewidth]{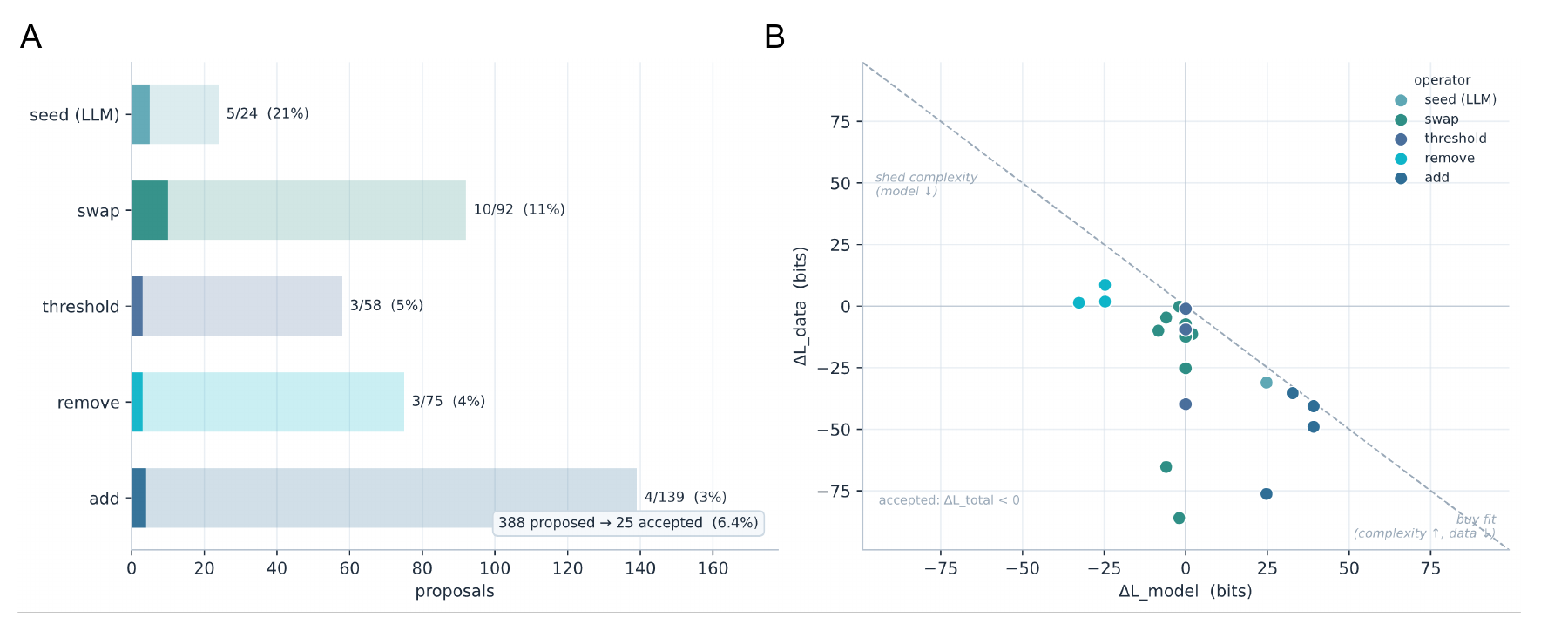}%
\caption{Anatomy of the MDL gate across the discovery run (data from \citep{buehler_break_world_2026}).
(A) Gate selectivity by proposal operator: the number of proposals accepted
versus proposed, aggregated over all iterations. Of $388$ proposals only $25$
are accepted ($6.4\%$), and the acceptance rate is strongly operator-dependent
(structure-recombining moves survive most often (seed $21\%$, swap $11\%$)
while bare feature additions rarely do (add $3\%$)). (B) The minimum-description-length
trade-off of each accepted move, shown as its change in model code length
$\Delta L_{\mathrm{model}}$ versus residual data code length $\Delta L_{\mathrm{data}}$;
deltas are computed within an iteration (between consecutive accepted states, where
the evidence set is fixed), giving $21$ of the $25$ accepted moves. All accepted
moves fall below the $\Delta L_{\mathrm{total}}=0$ frontier (dashed). Operators
occupy distinct regions: additions and seeds pay model bits to buy data bits
(lower right; mean $\Delta L_{\mathrm{model}}=+33.9$, $\Delta L_{\mathrm{data}}=-50.3$
for additions), threshold tuning improves fit at essentially zero model cost
(near-vertical), swaps give model-neutral fit gains (the most frequent accepted
move), and removals shed model complexity at a small data-fit cost (upper left;
mean $\Delta L_{\mathrm{model}}=-27.4$, $\Delta L_{\mathrm{data}}=+3.9$). Removals
are thus a high-yield operator rather than rare cleanup, showing that accepted
discovery includes retraction and compression, not only accumulation.
}
\label{fig:gate_anatomy}
\end{figure}

\begin{figure}[ht]
\centering
\includegraphics[width=\linewidth]{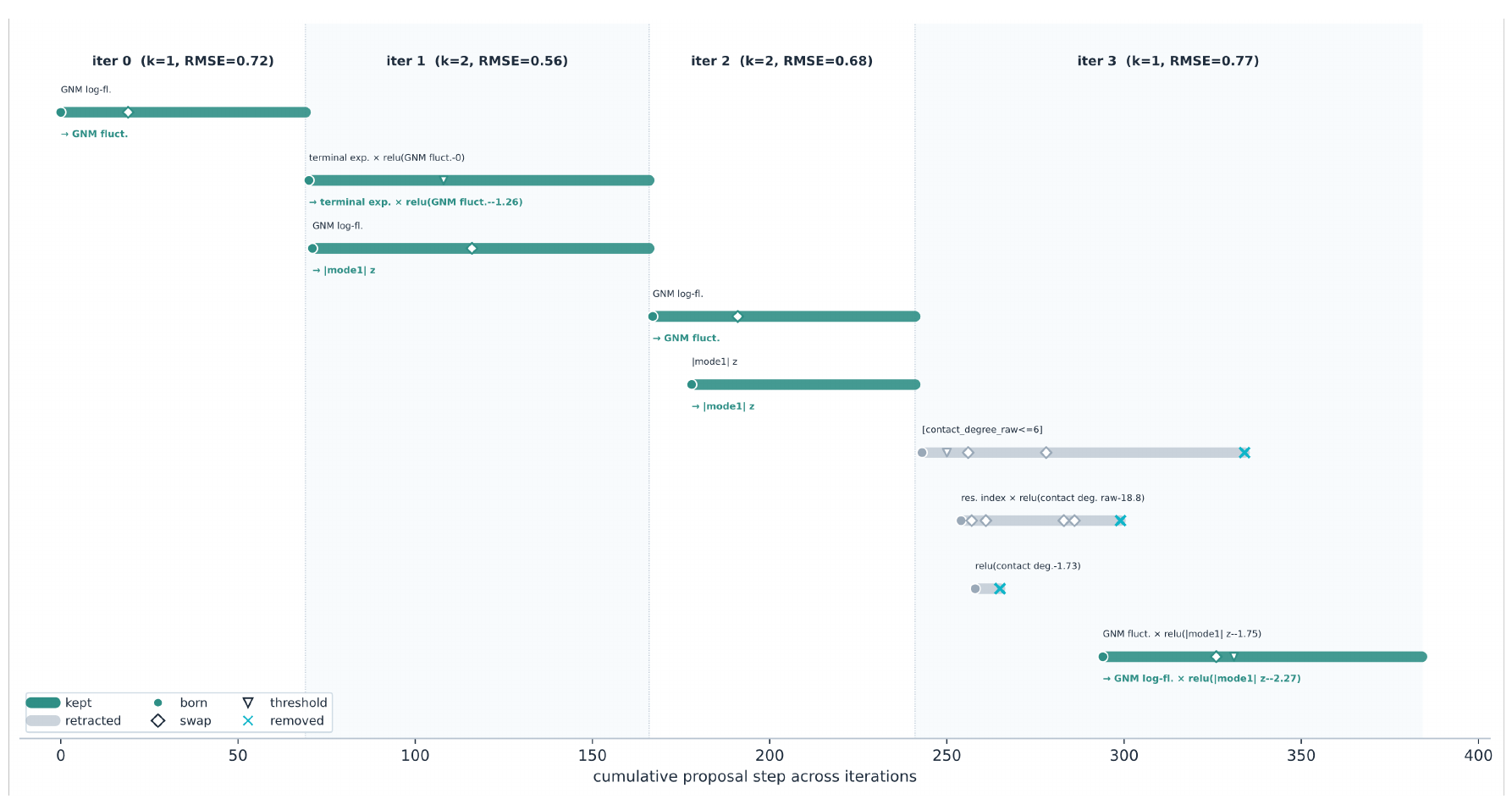}%
\caption{Feature lifecycle across the discovery run, computed from the
accepted moves of each iteration's inner search. Each horizontal bar is a
feature slot, with markers for its introduction (born, by an add or seed move),
factor swaps, threshold tunings, and removal. Slots present at the end of an
iteration are kept (teal); slots removed before then are retracted (grey), with
the model commitment dropped while its evidence is preserved. The label below
each kept slot gives its final symbolic form. The horizontal axis is the
cumulative proposal step across iterations; bands give the iteration index,
active-feature count $k$, and RMSE. Iterations $0$--$2$ only accrete features,
whereas iteration~$3$ explores four slots and retracts three, collapsing to the
single surviving mode-conditioned law ($k=1$), so the accepted model is reached
by pruning and compression rather than monotone accumulation. Thresholded
factors are shown in the internal $\mathrm{relu}(x-\tau)$ form, with fitted
$\tau=-\theta$ relative to the $[z+\theta]_+$ parametrization of Eq.~\ref{eq:bfactor_spectral_form}.}
\label{fig:feature_lifecycle}
\end{figure}

\subsection{CategoryScienceClaw adds a categorical layer to ScienceClaw}
\label{sec:categoryscienceclaw_results}
\hlyellow{CategoryScienceClaw is the second case framework in this paper, and is a wrapper around ScienceClaw. ScienceClaw is the agentic execution substrate where it organizes scientific work as skill-mediated production of artifacts, keeps metadata and parent lineage for those artifacts, exposes unresolved questions as shared open needs, uses pressure and feedback signals to decide what agents should attempt next, and allows active workflows to mutate as evidence changes. In the ScienceClaw $\times$ Infinite architecture, Infinite adds the communication substrate (Fig.~}\ref{fig:scienceclaw_infinite}\hlyellow{): structured posts, hypotheses, methods, findings, claim links, votes, comments, reputation, and moderation. This makes the combined system not merely a tool-calling loop but an artifact-centered scientific workflow and discourse system whose computation and communication traces can both be audited.}

\hlyellow{CategoryScienceClaw adds the explicit categorical and proof-carrying layer to that substrate. Skills become morphism signatures, artifacts become typed objects with content hashes and parents, open needs become typed holes to be filled by compatible morphisms, worker heartbeats become decentralized reactions, and certificates/audits check type and provenance validity. The layer is domain-general; fatigue crack growth in high-entropy alloys is the worked scientific example used here because its descriptor choice, model-selection gate, and validation caveat can be displayed compactly. In the HEA run studied here, this layer records materials records, phase labels, fatigue descriptors, candidate screening rules, accepted and rejected model artifacts, gate records, regime-transition audits, and synthesized figure/report artifacts. Thus the scientific object is not only a final plot or a fluent written claim; it is a typed discovery graph whose morphisms connect scientific inputs to gate-checked interpretations.}

\begin{figure}[h]
\centering
\begin{tikzpicture}[
  box/.style={draw,rounded corners=2pt,fill=blue!6,minimum width=2.65cm,minimum height=0.72cm,align=center,font=\scriptsize},
  gate/.style={draw,rounded corners=2pt,fill=orange!12,minimum width=2.65cm,minimum height=0.72cm,align=center,font=\scriptsize},
  disc/.style={draw,rounded corners=2pt,fill=green!10,minimum width=2.65cm,minimum height=0.72cm,align=center,font=\scriptsize},
  arr/.style={-{Stealth[length=2mm]},line width=0.55pt}
]
\node[box] (skills) at (0,0) {typed skill\\registry};
\node[box] (agents) at (3.2,0) {agents select\\skills};
\node[box] (arts) at (6.4,0) {immutable\\artifacts};
\node[gate] (reactor) at (6.4,-1.45) {ArtifactReactor\\pressure + overlap};
\node[box] (index) at (3.2,-1.45) {global index\\open needs};
\node[gate] (mut) at (0,-1.45) {mutation\\active graph};
\node[disc] (inf) at (9.6,-0.7) {Infinite\\posts + feedback};
\draw[arr] (skills) -- (agents);
\draw[arr] (agents) -- (arts);
\draw[arr] (arts) -- (reactor);
\draw[arr] (reactor) -- (index);
\draw[arr] (index) -- (mut);
\draw[arr] (mut) -- (skills);
\draw[arr] (arts) -- (inf);
\draw[arr] (inf) |- (reactor);
\end{tikzpicture}
\caption{ScienceClaw $\times$ Infinite as a distributed typed artifact system. ScienceClaw executes typed skill compositions and records immutable lineage; the ArtifactReactor and mutation layer coordinate active search; Infinite turns computational artifacts into public scientific discourse with feedback that can re-enter the discovery loop.}
\label{fig:scienceclaw_infinite}
\end{figure}
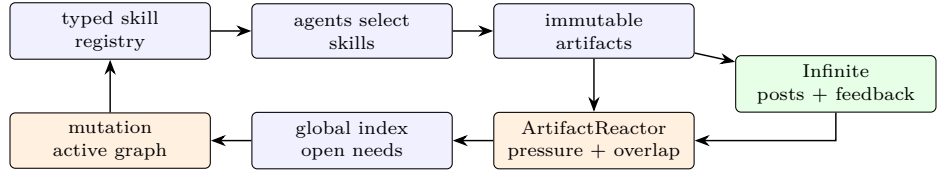

The implemented structure should be read precisely. Let $A_t$ be the finite set of ScienceClaw artifacts visible at time $t$, and let
\[
\tau:A_t\longrightarrow T
\]
assign each artifact its recorded artifact type. At minimum, this gives a family
\[
I_t(X)=\{a\in A_t:\tau(a)=X\},\qquad X\in T,
\]
that is, a copresheaf over the discrete category of implemented artifact types. Each artifact $a$ also records a skill label $\sigma(a)$, a producer agent, a content hash, and a parent list $p(a)=(a_1,\ldots,a_k)$. Thus the realized computational record is not merely a set of files but a typed acyclic hypergraph whose generating operations have the form
\[
(a_1,\ldots,a_k)\xrightarrow{\ \sigma(a)\ } a .
\]
The unary shadow of this hypergraph generates a free provenance category $\Prov_t$; the full multi-parent synthesis structure is more accurately read as a typed multicategory or colored-operadic provenance record. This is weaker than a software-enforced schema category, and should not be overclaimed. It is nevertheless already enough to support the central categorical interpretation: realized scientific work is a typed population of artifacts together with composable, auditable generating operations.

Open needs add the missing-hole structure. A need signal specifies a desired artifact type, query, rationale, and optional preferred skills. In categorical terms, it marks an unfilled target object or missing cone in the active provenance diagram. The ArtifactReactor does not yet solve a formal Kan-extension or lifting problem; it performs the implemented engineering analogue, using pressure scores and schema overlap to find artifacts and skills that may complete the diagram. Mutation then changes the active subgraph by marking artifacts active, rejected, superseded, or pruned while retaining immutable lineage.

Infinite extends this provenance record into scientific discourse. A post is not just text; it is a typed claim artifact with hypothesis, method, findings, evidence, links, and feedback. Links such as extension, contradiction, replication, or citation are discourse morphisms. Votes, comments, and reputation are verifier signals. The current integration therefore defines an implemented publication map from ScienceClaw artifacts and parent relations into Infinite posts, artifact records, and discourse links. It is not yet a certified functor in software, but it has the functorial shape required for one: computational lineage is carried into a public, inspectable, revisable scientific record.

\hlyellow{The CategoryScienceClaw worked case is a fatigue-crack-growth discovery test on an open high-entropy-alloy database. The scientific question is whether fatigue-crack-growth resistance can be screened by the scalar threshold $\Delta K_{th}$ alone, or whether threshold resistance must be separated from post-threshold crack-growth stability measured by the Paris slope $m$. CategoryScienceClaw records the threshold-only interpretation as a rejected audit object and commits a two-axis damage-tolerance map as the accepted hypothesis for this public-data case.}

%\FloatBarrier

\subsection{CategoryScienceClaw HEA fatigue-crack-growth} as a typed discovery graph
\label{sec:categoryscienceclaw_mechanics}

\hlyellow{CategoryScienceClaw makes the categorical framework operational at the level of individual scientific claims while retaining the larger ScienceClaw discovery-system structure of typed skills, immutable lineage, pressure coordination, workflow mutation, and public discourse. The CategoryScienceClaw application is a high-entropy-alloy (HEA) fatigue-crack-growth investigation because it has direct mechanics and materials-science content: it tests whether a scalar threshold descriptor is sufficient for fatigue screening, or whether the accepted representation must separate threshold resistance from post-threshold crack-growth stability. The input is the open Materials Cloud HEA fatigue database, which includes fatigue-crack-growth-rate records with composition, phase constitution, processing, tensile properties, test conditions, $\Delta K_{th}$, and Paris slope $m$ \citep{chen_hea_fatigue_database_2022,chen_hea_fatigue_dataset_2022}. The generated replay files include the cleaned FCGR table, typed artifact payloads, the discovery-test JSON, a categorical discovery graph, and the figure/report artifacts.}

\hlyellow{The mechanics formalism is explicit. In the Paris regime, fatigue-crack-growth rate is represented by}
\[
\displaystyle \frac{da}{dN}=C(\Delta K)^m,
\]
\hlyellow{where $\Delta K$ is the stress-intensity-factor range, $m$ is the Paris slope, and $\Delta K_{th}$ is the threshold below which long-crack growth is conventionally treated as negligible. A threshold-only screening rule ranks candidates by}
\[
\displaystyle M_0:\qquad Q_{\rm th}=\Delta K_{th}.
\]
\hlyellow{This rule treats a larger threshold as uniformly better. The accepted model is instead a two-axis damage-tolerance descriptor}
\[
\displaystyle M_1:\qquad Q_{\rm DT}=\bigl(\Delta K_{th},m,\phi\bigr),
\]
\hlyellow{where $\phi$ records phase class. The scientific reason is that $\Delta K_{th}$ and $m$ represent different risks: the former describes resistance to crack-growth onset, while the latter describes acceleration after threshold crossing.}

\hlyellow{For records with both $\Delta K_{th}$ and $m$, CategoryScienceClaw finds that FCC+BCC HEA records have a higher median threshold than FCC records, $7.0$ versus $4.8~\mathrm{MPa}\sqrt{\mathrm{m}}$, but also a much steeper median Paris slope, $5.9$ versus $3.155$. Thus the group that looks favorable under $M_0$ is less favorable once post-threshold crack-growth stability is represented. The declared gate accepts the two-axis descriptor when a phase group has higher median $\Delta K_{th}$ than FCC records but at least $50\%$ higher median Paris slope. The FCC+BCC group passes this gate: the median threshold ratio is $1.46$ and the median slope ratio is $1.87$.} (Fig.~\ref{fig:csc_hea_fatigue})

\hlyellow{Categorically, the run is a typed discovery move. The old fatigue-screening regime contains HEA records, phase labels, and scalar threshold values. The enlarged regime contains the Paris-slope object, threshold--slope paired descriptor, rejected threshold-only model, accepted two-axis damage-tolerance map, gate record, and science claim. In the notation of} Definition~\ref{def:transition}, \hlyellow{the objectwise discovery-residual diagnostic is}
\[
R(A') =
I'_{t+1}(A')\setminus \mathrm{im}(\bar\rho_{A'}).
\]
\hlyellow{For the HEA fatigue run, the transported objects include composition, phase class, testing condition, and $\Delta K_{th}$ records. The residual objects include Paris-slope stability, the threshold--slope map, and the high-threshold/high-slope risk class. This residual content is the scientific value of the case: it changes the design-relevant descriptor from ``maximize threshold'' to ``separate threshold resistance from post-threshold growth stability.''}

\begin{figure}[ht]
\centering
\includegraphics[width=\linewidth]{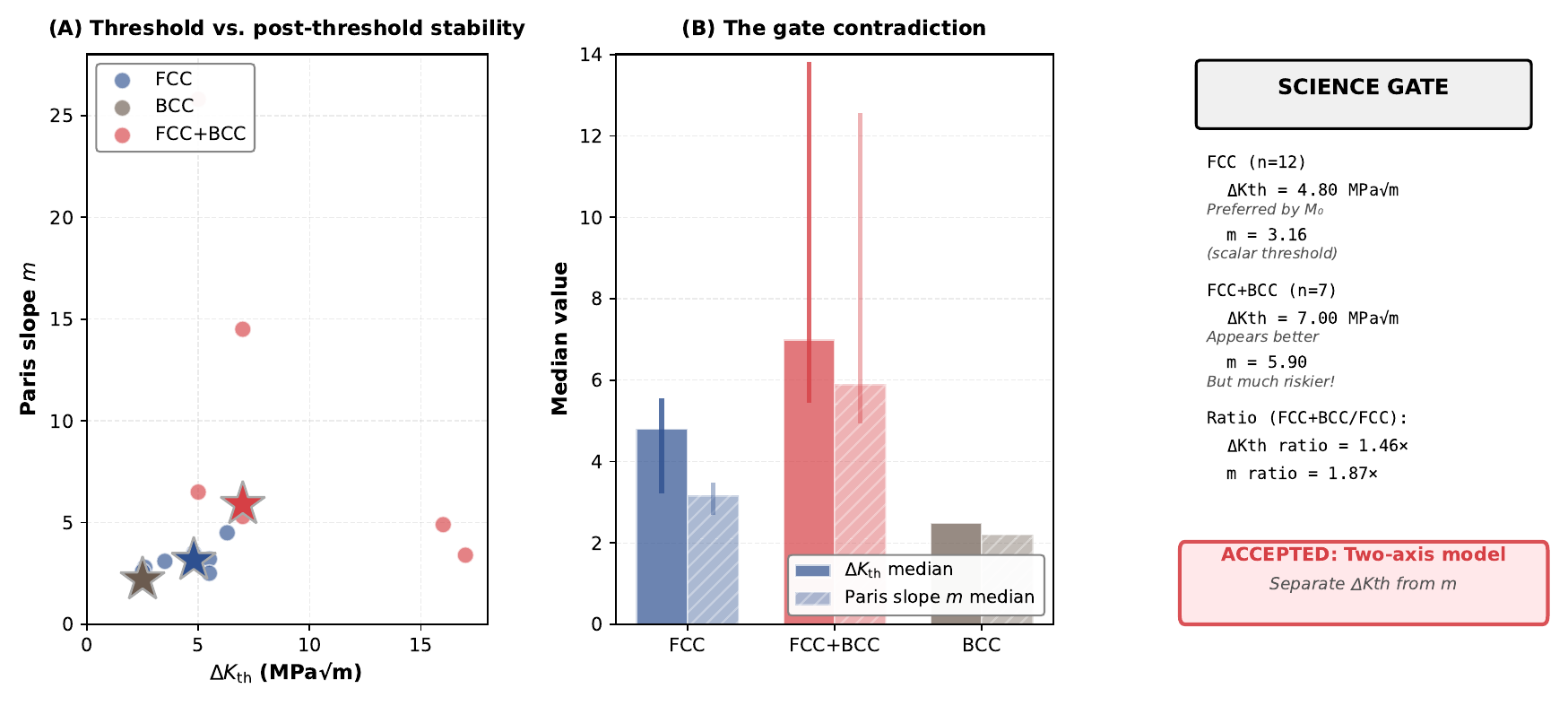}%
\caption{\hlrev{CategoryScienceClaw high-entropy-alloy fatigue-crack-growth investigation. The figure maps fatigue threshold $\Delta K_{th}$ against Paris slope $m$, compares median threshold and slope by phase class, and records the gate by which a threshold-only ranking is rejected in favor of a two-axis damage-tolerance map. The scientific hypothesis is that HEA fatigue screening should separate crack-growth threshold from post-threshold growth stability; high-threshold/high-slope dual-phase records are a distinct risk class requiring prospective validation.}}
\label{fig:csc_hea_fatigue}
\end{figure}

\hlyellow{This mechanics gate is close in spirit to the Builder/Breaker MDL gate but differs in what is being charged. Builder/Breaker accepts symbolic DAG revisions when paired refitting reduces total description length on accumulated protein-mechanics evidence. CategoryScienceClaw accepts the HEA fatigue descriptor when the richer typed representation resolves a design-relevant contradiction in the public FCGR summary: high threshold does not imply low post-threshold growth sensitivity. In both cases rejected alternatives remain first-class audit objects rather than disappearing from the record.}

\paragraph{A concrete instance: CategoryScienceClaw HEA fatigue-crack-growth}.

\hlyellow{A compressed provenance diagram is}
\[
\displaystyle\begin{gathered}
a_{\rm HEA}\xrightarrow{\rm threshold}a_{\Delta K_{th}},\qquad
a_{\rm HEA}\xrightarrow{\rm Paris\ fit/record}a_m,
\qquad
a_{\rm HEA}\xrightarrow{\rm phase\ typing}a_\phi,\\
(a_{\Delta K_{th}},a_m,a_\phi)\xrightarrow{\rm gate}(M_0,M_1,a_{\rm accepted}),\qquad
(a_{\rm accepted},M_0,M_1)\xrightarrow{\rm synthesis}a_{\rm claim}.
\end{gathered}
\]
\hlyellow{The HEA fatigue case exemplifies discovery in the formal sense: evidence drives a regime transition, and residual content is mechanically novel. The system tests whether fatigue resistance can be screened by threshold $\Delta K_{\rm th}$ alone or requires a two-axis model pairing threshold with Paris slope $m$. Under threshold-only screening, FCC+BCC dual-phase HEAs appear favorable (median $\Delta K_{\rm th}=7.0$ vs. 4.8 MPa$\sqrt{\mathrm{m}}$ for FCC). Yet when slope is added, the same materials show median $m=5.9$ versus 3.16, revealing a risk class the scalar descriptor cannot accommodate. The gate rejects single-axis screening in favor of a two-axis damage-tolerance model validated against 180+ records in the open Materials Cloud database. The resulting principle resolves a design constraint in materials science: threshold and post-threshold crack-growth stability are decoupled. The framework treats this transition not as a retrospective narrative but as a prospective design criterion: any HEA with elevated threshold and elevated Paris slope requires mechanical validation before acceptance, establishing a testable prediction for future alloys.}

\hlyellow{Supporting audits of 7T10 structure-contact mechanics, mechanobiology force paths, membrane biophysics, and fiber-network mechanics are reported in} Figs.~S1--S4.
%\FloatBarrier

\subsection{Instantiations across agentic systems}
\label{sec:instantiations}

The same formal pattern appears across several agentic systems developed in this research line, with different regimes, artifact populations, and gates. Table~\ref{tab:systems} summarizes the mapping. The table is deliberately not a leaderboard. It is a structural comparison of how different systems instantiate typed artifact states, fixed-regime updates, and regime-enlarging mechanisms.

\begin{table}[!htbp]
\centering
\footnotesize
\setlength{\tabcolsep}{3pt}
\begin{tabular}{@{}>{\raggedright\arraybackslash}p{0.17\linewidth}>{\raggedright\arraybackslash}p{0.25\linewidth}>{\raggedright\arraybackslash}p{0.25\linewidth}>{\raggedright\arraybackslash}p{0.23\linewidth}@{}}
\toprule
System & Main artifact regime & Gate or verifier & Discovery-relevant mechanism \\
\midrule
ProtAgents \citep{ghafarollahi_protagents_2024} & protein sequences, structures, force and dynamics artifacts & agent critique plus physics and ML tool outputs & physics simulators expose failure modes of generated designs \\
GraphAgents~\cite{stewart2026graphagentsknowledgegraphguidedagentic} & layered material knowledge graphs and substitution candidates & feasibility, manufacturability, and cross-source consistency & cross-layer composition of evidence from heterogeneous corpora \\
Sparks / SciAgents \citep{ghafarollahi_sparks_2025,ghafarollahi_sciagents_2025} & research goals, hypotheses, code, results, and reports & novelty, interestingness, executable results & ideation-to-experiment loops over a research-artifact schema \\
Builder/Breaker \citep{buehler_break_world_2026} & protein-mechanics evidence and symbolic DAG world models & paired MDL compression & adversarial evidence forces retractions and regime enlargement \\
ScienceClaw $\times$ Infinite with CategoryScienceClaw layer \citep{wang_scienceclaw_infinite_2026} & execution artifacts and open needs; discourse posts, claim links, feedback; categorical objects, morphisms, lineage, gates, figures, and reports & pressure scoring, schema overlap, mutation status, discourse feedback, reputation, and proof certificates & plannerless artifact exchange plus public claim verification, lifted into typed categorical provenance \\
% ScienceClaw $\times$ Infinite \citep{wang_scienceclaw_infinite_2026} & typed skills, immutable artifacts, open needs, and discourse artifacts & pressure, schema overlap, mutation, feedback, reputation & plannerless artifact exchange and public claim-level verification \\
\bottomrule
\end{tabular}
\caption{Agentic systems as typed artifact systems. Each system has a schema, artifact population, gate, and mechanism by which the active regime can be stressed or enlarged.}
\label{tab:systems}
\end{table}

%\FloatBarrier

\subsection{Relation to existing formalisms}
\label{sec:related}

The problem addressed here also has a historical and philosophical lineage. Bacon and Whewell treated scientific method as an organized practice of collecting, comparing, and conceptually ordering phenomena; Peirce emphasized inquiry as a disciplined process for stabilizing belief under doubt \citep{bacon_novum_organum_1620,whewell_inductive_sciences_1840,peirce_fixation_1877}. Popper, Kuhn, and Lakatos then made failure, anomaly, paradigm change, and research-programme revision central to accounts of scientific progress \citep{popper_logic_1959,kuhn_structure_1962,lakatos_falsification_1970}.

Goethe and Whitehead supply a complementary lineage in which form, relation, and process are primary scientific objects rather than incidental descriptions \citep{goethe_metamorphosis_1790,whitehead_science_modern_world_1925}. Polanyi and Hacking are also relevant: scientific knowledge is not only explicit proposition but skillful practice, and experimental intervention is part of what makes a representation scientifically real \citep{polanyi_personal_knowledge_1958,hacking_representing_intervening_1983}. The contribution of the present paper is to give this broad intuition an auditable mathematical substrate for AI systems: typed artifacts, composable provenance, explicit gates, and verified regime transitions.

\subsubsection{Relation to scientific knowledge graphs and workflow provenance}
Scientific knowledge graphs and workflow-provenance models already provide important
ways to represent entities, relations, computational activities, and data products in a graph~\citep{provo2013, bechhofer2013linkeddata, davidson2008provenance, jaradeh2019orkg}. The present contribution
is therefore not the observation that computation can appear in a graph. Rather, it is to give
a categorical semantics in which knowledge, computation, verification, rejection, public
discourse, and schema-changing discovery are components of a single scientific state
(Definitions~\ref{def:regime}, \ref{def:copresheaf}, \ref{def:elements}, and~\ref{def:transition}). Computation is represented both at the type level, as a morphism
\(f:A\to B\), and at the artifact level, as a realized map
\(I_t(f):I_t(A)\to I_t(B)\) or, in the multicategorical case, as a multi-parent operation
(Proposition~\ref{prop:multicat_kan}). Discovery is then not graph completion inside
a fixed ontology, but a verified regime transition together with residual content beyond
functorial transport (Proposition~\ref{prop:kan_obstruction}).

\begin{table}[t]
\centering
\small
\setlength{\tabcolsep}{6pt}
\renewcommand{\arraystretch}{1.2}
\begin{tabular}{@{}p{0.18\linewidth}p{0.36\linewidth}p{0.38\linewidth}@{}}
\toprule
\textbf{Aspect} &
\textbf{Knowledge graph / workflow-provenance view} &
\textbf{Categorical knowledge--computation graph} \\
\midrule

Vocabulary &
Fixed or slowly curated ontology of entity and relation types &
Schema category \(\Sschema_b\) of artifact types and admissible operations, with possible
transition \(u:\Sschema_b\to\Sschema_{b'}\) \\

Computation &
Activities, services, scripts, or workflow steps that consume inputs and generate outputs &
Typed morphisms \(f:A\to B\) and realized artifact maps
\(I_t(f):I_t(A)\to I_t(B)\), with multi-parent synthesis represented by a typed
multicategory or colored-operadic provenance object \\

Scientific commitment &
Validity is often represented as annotations, confidence/provenance metadata, or downstream checks &
Commitment, rejection, supersession, and stress testing are represented by gates \(\V_b\),
description-length or model-selection functionals \(L_b\), diagnostics, and retained
rejected artifacts \\

Public discourse &
Claims, citations, comments, and replications may be represented as graph nodes or
external communication records &
A discourse layer \(\mathsf D_t\) and publication map \(\pi_t\) carry audited artifact
paths or hyperpaths into claim, evidence, objection, replication, and validation objects \\

Dynamics &
Graph completion, link prediction, or workflow execution inside a fixed representational frame &
Fixed-regime search is \(\Phi_b\); discovery is a verified transition with transported
evidence \(\Lan_u I_t\) and comparison map
\(\bar\rho:\Lan_u I_t\to I'_{t+1}\) \\

Discovery content &
New facts, links, generated outputs, or completed workflow products &
Objectwise residual artifacts
\(I'_{t+1}(A')\setminus\mathrm{im}(\bar\rho_{A'})\), together with recorded
regime-level additions such as new object types, morphisms, tools, verifiers, or grammar
productions needed to express them \\
\bottomrule
\end{tabular}
\caption{From knowledge graphs and workflow provenance to categorical knowledge--computation graphs.}
\label{tab:kg_to_ckg}
\end{table}

\subsubsection{Relation to categorical learning, coalgebra, and model selection}
The framework is adjacent to, but distinct from, categorical deep learning. Categorical accounts of neural networks and learning, including backpropagation as a functor and categorical deep learning programs based on parametric maps and lenses, clarify the structure of model architectures and training procedures \citep{fong_spivak_tuyeras_2019,gavranovic_cdl_2024,cruttwell_parametric_lenses_2024,crescenzi_zanasi_survey_2024}. The present paper instead studies systems whose state is a typed scientific artifact population and whose most important operation may be a change of regime. The unit of analysis is therefore not only a trained model, but a scientific workflow capable of altering its own schema of admissible artifacts.

The framework also touches coalgebra and dynamical systems. A fixed-regime update resembles a state transition system and can be studied with coalgebraic ideas such as bisimulation and final behavior \citep{aczel_mendler_1989,rutten_coalgebra_2000}. The difference is that discovery changes the state space itself: the relevant trajectory lives over a base of regimes rather than inside one fixed carrier. This is why the indexed or fibered picture is natural.

We finally note that the gate connects the framework to MDL, Bayesian model selection, algorithmic information, and open-endedness. MDL and Solomonoff-style compression provide one rigorous account of why a simpler world model that explains broader evidence is preferable \citep{rissanen_mdl_1978,gruenwald_mdl_2007,solomonoff_1964,hutter_aixi_2005,schwarz_bic_1978}. Open-endedness and automated-scientist systems study how systems escape fixed objectives or generate new scientific tasks \citep{stanley_open_endedness_2017,wang_poet_2019,lu_ai_scientist_2024,lu_ai_scientist_v2_2025}. Scientific machine learning supplies many of the typed operators that appear as morphisms in our schemas, from physics-informed learning to operator learning and equation-free multiscale computation (e.g., \citep{Gu2018EML,karniadakis_pinns_2021,lu_deeponet_2021,kevrekidis_equation_free_2003}). Similarly, the sparse identification of governing equations~\cite{doi:10.1073/pnas.1517384113} and the extraction of symbolic models via inductive biases~\cite{cranmer2020discoveringsymbolicmodelsdeep} provide established paradigms for discovering interpretable physical laws, which appear in our framework as composable, complexity-penalized morphisms.
The contribution here is to place these components into one typed, provenance-preserving, regime-changing account of discovery.

\subsection{Design principles and open problems}
\label{sec:open_results}

The framework suggests five open problems articulated as follows:

\begin{openproblem}[Convergence on growing regimes]
Classical fixed-point theory assumes a fixed underlying space. Discovery systems iterate $\Phi_b$ while also producing regime transitions $b_0\to b_1\to b_2\to\cdots$. Under what conditions does the sequence of transported artifact states converge to a stable object in an appropriate colimit or indexed category? When is non-convergence productive exploration, and when is it unproductive oscillation?
\end{openproblem}

\begin{openproblem}[Scaling laws for discovery]
Standard scaling laws measure loss or benchmark performance inside a fixed regime. Discovery requires a different observable: the rate, quality, and accepted value of regime enlargements. How do discovery rates scale with model size, tool diversity, simulator fidelity, evidence heterogeneity, schema richness, and verifier strength?
\end{openproblem}

\begin{openproblem}[Verification tooling for agentic loops]
The framework requires provenance preservation and gate compatibility. Practical systems need tooling that can replay downstream artifact chains, check approximate naturality, account for description length or pressure scores, and record retractions without deleting provenance.
\end{openproblem}

\begin{openproblem}[Learning the base schema category]
The framework assumes a schema category $\Sschema_b$. Current systems construct it by engineering choice. A major open problem is to learn useful schema categories from corpora, tool signatures, code, figures, equations, and laboratory protocols while keeping morphisms scientifically meaningful. Functorial data migration, olog-style modeling, and category-theoretic accounts of hierarchical materials and building-block replacement provide a starting toolkit \citep{spivak_schemas_2012,buehler_spivak_2011,giesa_spivak_2011_patterns,giesa_spivak_2012_buildingblock}; the open problem is computational: how to learn $\Sschema_b$ and an accompanying description-length functional $L_b$ from data at the scale of working scientific corpora.
\end{openproblem}

\begin{openproblem}[Multicategorical discovery]
The Kan audit in Section~\ref{sec:builder_breaker_results} distinguishes generator-level transport from composite reachability through newly admitted multi-input morphisms. Proposition~\ref{prop:multicat_kan} gives the corresponding multicategorical obstruction under the standard operadic Kan-extension setup. The open problem is now computational and statistical: how can an agentic system learn the typed multicategory or colored-operadic schema $M_b$ from traces, tool signatures, equations, figures, and artifact lineages, and how can it estimate the artifact and operation components of discovery cost at scale?
\end{openproblem}

% =====================================================================
\section{Conclusions and Outlook}
\label{sec:conclusions}
% =====================================================================

This paper has argued that agentic scientific discovery requires two levels of structure. Inside a fixed regime, an agentic system updates typed artifact populations: it adds data, models, simulations, hypotheses, critiques, and reports while preserving provenance. At that level, copresheaves, categories of elements, natural transformations, and endofunctorial dynamics give a precise language for structures already present in working discovery systems. Discovery in the stronger sense occurs when evidence forces a transition to a new regime, and left Kan extension gives the mathematical language for transporting old artifacts into the enlarged vocabulary.

The Builder/Breaker protein-mechanics case shows this structure quantitatively: a symbolic world model is revised by adversarial evidence and an MDL gate, with rejected edits and retractions visible in the audit trail. 
\hlyellow{CategoryScienceClaw shows the same structure as a proof-carrying categorical layer over the ScienceClaw $\times$ Infinite substrate: typed skills, immutable lineage, open needs, pressure-based coordination, and public discourse become typed objects and morphisms. The HEA fatigue-crack-growth run records a materials-science model revision as inspectable provenance: the rejected threshold-only ranking, accepted two-axis damage-tolerance map, gate record, generated figure, and prospective validation claim are retained as typed artifacts.}

More broadly, the work suggests a reciprocal program: AI can accelerate mechanics, but mechanics can also
discipline AI by providing concepts of state, load, failure, invariance, admissible transformation, and
constitutive closure for self-revising discovery systems.

\subsection*{Toward an executable categorical discovery substrate}
\hlyellow{CategoryScienceClaw is already close to the categorical picture because it makes the ScienceClaw data model and workflow discipline explicit. It has typed artifact metadata, skill labels, immutable parent lineage, content hashes, open needs, pressure-ranked reactions, mutation of active status, public discourse objects with feedback signals, and domain-specific gate records. The categorical layer adds object and morphism signatures, typed needs, proof certificates, and audit/replay checks. This is enough to read the current implementation as a typed artifact family equipped with an acyclic provenance hypergraph and a publication map into a discourse graph. The HEA fatigue case in this paper is a scientific audit within that broader substrate: the threshold-only descriptor is preserved as a rejected model, the threshold--slope descriptor is preserved as the accepted representation, and the high-threshold/high-slope class is preserved as a testable materials-design claim. Supporting mechanics cases remain in the SI.}

The next step is therefore not to replace ScienceClaw, but to lift structures already present in ScienceClaw and CategoryScienceClaw into explicit mathematical objects. Skill manifests would become morphism signatures with declared input objects, output objects, constraints, and verifiers. Artifact stores would become materialized fibers of a copresheaf over that schema. Multi-parent synthesis would be treated as a typed multicategory or operadic composition, rather than only as a parent list. Open needs would become typed holes or lifting problems. Public discourse would become a category whose objects are claims, posts, artifacts, evidence bundles, objections, replications, and validation signals. The CategoryScienceClaw publication map would then be a checked map from provenance to discourse, preserving identities, parent-child composition, and validation status.

Making this explicit would change the operational status of the platform. The system would no longer merely display provenance; it could verify that provenance diagrams commute, that every public claim has an admissible artifact path, that retractions and supersessions preserve old evidence, and that new artifact types are introduced through recorded schema transitions. In the language of this paper, ScienceClaw equipped with the CategoryScienceClaw layer would move from an operational typed-artifact platform to an executable categorical discovery substrate. The impact would be practical: stronger audit trails, machine-checkable claim lineage, better comparison among autonomous investigations, and a route toward measuring discovery as the residual content added beyond transport of prior evidence.

The central insight is therefore that search is iteration inside a typed scientific regime; discovery is a verified regime transition---enlargement, restructuring, or compression---together with the residual content that lies outside functorial transport of prior evidence.
This statement is mathematical enough to constrain future theory and concrete enough to guide the design of future AI discovery systems.

% =====================================================================
\section{Materials and Methods}
\label{sec:methods}
% =====================================================================

\subsection{Regimes, schemas, and copresheaves}

\begin{definition}[Discovery regime]
\label{def:regime}
A \textbf{discovery regime} is a tuple
\[
b=(\Sschema_b,\Gamma_b,\V_b,L_b),
\]
where $\Sschema_b$ is a small schema category of artifact types and allowed operations, $\Gamma_b$ is a grammar for composing admissible artifacts or model structures, $\V_b$ is a verifier or gate, and $L_b$ is an optional description-length functional. The functional $L_b$ may be presented either absolutely on artifact states or as a relative (conditional) length functional; we write $L_b(I\mid J)$ for the code length of $I$ given a subobject $J\hookrightarrow I$, with $L_b(I)\equiv L_b(I\mid\varnothing)$ recovering the absolute case. In this paper, $\Sschema_b$ carries the main type-and-operation structure; $\Gamma_b$, $\V_b$, and $L_b$ record the additional grammar, acceptance, and scoring data attached to that structure. When these additional data are clear from context, ``regime'' is used informally as shorthand for the schema and its associated commitments.
\end{definition}

\begin{definition}[Artifact-state copresheaf]
\label{def:copresheaf}
For a fixed regime $b$, an artifact state at time $t$ is a copresheaf
\[
I_t:\Sschema_b\to\Set .
\]
For each object $A\in\Obj(\Sschema_b)$, $I_t(A)$ is the set of artifacts of type $A$. For each morphism $f:A\to B$, $I_t(f):I_t(A)\to I_t(B)$ maps artifacts along an allowed operation.
\end{definition}

\begin{definition}[Category of elements]
\label{def:elements}
The realized provenance category of $I_t$ is the category of elements $\int_{\Sschema_b}I_t$. Its objects are pairs $(A,x)$ with $x\in I_t(A)$. A morphism $(A,x)\to(B,y)$ is a morphism $f:A\to B$ in $\Sschema_b$ such that $I_t(f)(x)=y$.
\end{definition}

\hlyellow{For partially realized scientific workflows, $I_t(f)$ may be a partial function, relation, stochastic kernel, or span rather than a total function. Numerical and statistical outputs may be accepted up to the tolerance, equivalence, or model-selection criterion declared by $\V_b$ or $L_b$. The total-function presentation is the clean base case; partiality can be handled by replacing $\Set$ with a category of partial maps, relations, spans, Kleisli morphisms, or typed records with status events.}

\subsection{Fixed-regime update and endofunctoriality}

Let $\K_b$ be a category whose objects are artifact-state copresheaves over $\Sschema_b$ and whose morphisms are provenance-preserving refinements. In the strict presentation used for the structural observations and the Kan proposition, a morphism $\alpha:I_t\to J_t$ is a componentwise injective natural transformation whose components $\alpha_A:I_t(A)\to J_t(A)$ preserve artifact identity, type, and provenance metadata, up to the equivalence or tolerance declared by $\V_b$. The injectivity assumption is the mathematical version of a practical audit rule: a refinement may add, annotate, or supersede artifacts, but it may not silently identify two previously distinct accepted artifacts.

\begin{definition}[Fixed-regime update]
\label{def:fixedupdate}
A \textbf{fixed-regime agentic update} is an endomap on artifact states,
\[
\Phi_b:\Obj(\K_b)\to\Obj(\K_b),
\]
that becomes an endofunctor $\Phi_b:\K_b\to\K_b$ when it maps refinement morphisms to refinement morphisms and preserves identities and composition.
\end{definition}
A realized committed trajectory will be written as a sequence
\[
I_t \xrightarrow{\delta_t} I_{t+1}=\Phi_b(I_t),
\]
where \(\delta_t\) is the refinement morphism that embeds the previous committed
ledger into the next committed ledger. Thus \(\Phi_b\) acts on states and on
refinements between alternative states, whereas \(\delta_t\) is the particular
accepted step in one realized run.

The update may be implemented by agents, simulators, retrieval systems, symbolic search, or human feedback. The categorical requirement is not that the implementation be deterministic or neural-free, but that the committed state preserve typed provenance. Concretely, an implementation earns the endofunctor notation only when each refinement $\alpha:I\to J$ induces a refinement $\Phi_b(\alpha):\Phi_b(I)\to\Phi_b(J)$ and these induced maps respect identity refinements and composition of refinements. In stochastic or asynchronous implementations, this condition is imposed on the committed-state relation, stochastic kernel, or Kleisli morphism rather than on every raw execution trace.

\subsection{Builder, Breaker, and gates}

\begin{definition}[Builder/Breaker system]
\label{def:bb}
Within a regime $b$, a Builder/Breaker system consists of a proposal mechanism $B$, an evidence-producing mechanism $K$, and a gate $\V_b$. The Breaker $K$ selects or generates evidence intended to stress the current artifact state. The Builder $B$ proposes candidate increments or revisions. The gate $\V_b$ decides whether the candidate is committed, rejected, superseded, or held for review.
\end{definition}

For the protein-mechanics case, the gate is MDL:
\begin{equation}
\V_b(M',M;D)=1
\quad\Longleftrightarrow\quad
L_{\mathrm{model}}(M')+L_{\mathrm{data}}(D\mid M')
<
L_{\mathrm{model}}(M)+L_{\mathrm{data}}(D\mid M).
\label{eq:mdl_methods}
\end{equation}
% Comparisons are paired: both models are refit on the same accumulated data $D$. For ScienceClaw $\times$ Infinite, $\V_b$ is pressure-based and social-technical: schema overlap, artifact status, conflict resolution, open-need satisfaction, discourse feedback, and reputation all contribute to commitment.
Comparisons are paired: both models are refit on the same accumulated data $D$. For CategoryScienceClaw, $\V_b$ combines mechanics-specific gates with platform-level commitment signals: candidate models are compared by model-selection scores and retained only when diagnostics pass, while pressure, mutation status, discourse feedback, and reputation help determine which artifacts remain active in the broader discovery graph.

\subsection{Verified regime transition and Kan transport}

\begin{definition}[Verified regime transition]
\label{def:transition}
Given an old state $I_t\in\K_b$ and a new accepted state $I'_{t+1}\in\K_{b'}$, a \textbf{verified regime transition} from $b$ to $b'$ relative to $(I_t,I'_{t+1})$ consists of a schema functor
\[
u:\Sschema_b\to\Sschema_{b'}
\]
together with a preservation natural transformation
\[
\rho:I_t\longrightarrow u^*I'_{t+1},
\]
where $u^*:[\Sschema_{b'},\Set]\to[\Sschema_b,\Set]$ is restriction along $u$. The state-dependence is part of the data: a verified transition is a piece of structure attached to a specific pair of states, not a universal property of the schemas alone. The transition satisfies four compatibility conditions:
\begin{enumerate}[leftmargin=1.7em,topsep=2pt,itemsep=2pt]
    \item $u$ is faithful on morphisms and injective on objects, except where the transition explicitly records a type split, quotient, or recoding;
    \item $\rho$ is componentwise injective and provenance-preserving, so old accepted artifacts remain inspectable after the transition;
    \item old commitments remain accepted after transport: for predicate-valued gates, $\V_b(I_t)=1$ implies $\V_{b'}(\mathrm{im}(\bar\rho))=1$, where $\bar\rho:\Lan_uI_t\to I'_{t+1}$ is the adjoint transpose of $\rho$.
    \item the new state is itself gate-accepted in its own regime: $\V_{b'}(I'_{t+1})=1$.
\end{enumerate}
If a description-length functional is present, the transported old state has no larger code length than the original old state, $L_{b'}(\mathrm{im}(\bar\rho))\le L_b(I_t)$, or else the increase is recorded as an explicit recoding cost. The transition is \textbf{nontrivial} when $\Sschema_{b'}$ contains a new object, morphism, verifier, grammar production, or tool class not generated inside $u(\Sschema_b)$, or when $I'_{t+1}$ contains accepted artifacts not in $\mathrm{im}(\bar\rho)$.
\end{definition}

The transported old artifact state is the left Kan extension
\begin{equation}
(\Lan_u I_t)(A') \cong \colim_{(uA\to A')\in(u\downarrow A')} I_t(A),
\label{eq:lan_formula}
\end{equation}
where the colimit exists because $\Set$ is cocomplete and $\Sschema_b$ is small. If $u$ is fully faithful, the unit $I_t\to u^*\Lan_u I_t$ is an isomorphism, so old fibers are preserved by transport alone. If $u$ is only faithful, if it recodes old object names, or if the new schema adds morphisms among old types, preservation is not automatic; it is supplied by the explicit map $\rho$ in Definition~\ref{def:transition}. When the comma category $(u\downarrow A')$ is empty---that is, when $A'$ receives no morphisms from $u(\Sschema_b)$---the colimit is the initial object of $\Set$, giving $(\Lan_u I_t)(A')=\varnothing$. New evidence then populates such fibers, or activates new morphisms absent from the old regime.

\subsection{Structural observations and the Kan obstruction}

The first three statements are closure and verification observations: they unpack what is already forced by a fixed schema, provenance-preserving refinement, and a declared gate. The first nontrivial categorical obstruction is the Kan statement that follows, where the comparison map makes residual content unavoidable. The later propositions record the functorial consequences needed for auditability, composition of transitions, and the multicategorical reading of product-like discovery moves.

\begin{observation}[Fixed-regime reachability]
\label{obs:reachability}
Let $b$ be fixed and let $\Phi_b:\K_b\to\K_b$ be an endofunctor on artifact states over $\Sschema_b$. For any state $I_0$ and any $n\in\N$, every object appearing in the provenance category $\int \Phi_b^n(I_0)$ lies over an object of $\Sschema_b$. Thus finite iteration of a fixed-regime update cannot by itself create an artifact of a type outside $\Sschema_b$.
\end{observation}

\begin{proof}
Each state in $\K_b$ is by definition a copresheaf on $\Sschema_b$. Applying $\Phi_b$ returns another object of $\K_b$, hence another copresheaf on the same schema. The category of elements of any such copresheaf has objects $(A,x)$ with $A\in\Obj(\Sschema_b)$. Induction on $n$ gives the claim.
\end{proof}

\begin{observation}[Verification requires provenance preservation and a gate]
\label{obs:verification}
Let $\alpha:I\to J$ be a committed update inside a fixed regime $b$. If the update is verified in the sense of the regime, then (i) $\alpha$ preserves prior provenance up to the equivalence or tolerance declared by $\V_b$, and (ii) $\alpha$ passes the gate $\V_b$. Conversely, if both conditions hold and $\V_b$ is the regime's declared acceptance criterion, then the update is verified.
\end{observation}

\begin{proof}
If provenance preservation fails, then at least one previously accepted artifact chain cannot be re-expressed after the update, so the update is not verified. If the gate fails, the update violates the regime's declared commitment criterion. Conversely, provenance preservation keeps old accepted diagrams valid, and the gate supplies the regime-specific acceptance judgment.
\end{proof}

\begin{observation}[Nontrivial discovery requires regime-extending structure]
\label{obs:external}
Let $b$ be fixed. A nontrivial discovery move that produces an artifact over a type, morphism, verifier, grammar production, or tool class absent from $\Sschema_b$ cannot be obtained by finite iteration of $\Phi_b:\K_b\to\K_b$ alone. It requires a verified regime transition $u:\Sschema_b\to\Sschema_{b'}$ or an equivalent schema-generating mechanism.
\end{observation}

\begin{proof}
By Observation~\ref{obs:reachability}, finite iteration of $\Phi_b$ yields artifact states over $\Sschema_b$. Hence it cannot create an artifact whose type is not an object of $\Sschema_b$. Moreover, because $b=(\Sschema_b,\Gamma_b,\V_b,L_b)$ is fixed during iteration, the grammar, verifier, and available tool classes are fixed as well. A move requiring a new type, morphism, verifier, grammar production, or tool class therefore cannot be the result of iteration inside $\K_b$ alone; it requires a transition to a regime in which the new structure exists.
\end{proof}

\begin{proposition}[Kan obstruction to transport-only discovery]
\label{prop:kan_obstruction}
Let $u:\Sschema_b\to\Sschema_{b'}$ and $\rho:I_t\to u^*I'_{t+1}$ be a verified regime transition (Definition~\ref{def:transition}). Let
\[
\bar\rho:\Lan_u I_t\longrightarrow I'_{t+1}
\]
be the unique adjoint transpose of $\rho$ under $\Lan_u\dashv u^*$. Three statements follow.
\begin{enumerate}[leftmargin=1.7em,topsep=2pt,itemsep=2pt]
    \item If $(u\downarrow A')=\varnothing$, then transport is empty at $A'$:
    \[
    (\Lan_u I_t)(A')=\varnothing .
    \]
    If $I'_{t+1}(A')\neq\varnothing$, then $\mathrm{im}(\bar\rho_{A'})=\varnothing$ and the residual at $A'$ is forced to be nonempty. Every artifact accepted at such an $A'$ must be supplied from outside transport of the old regime.
    \item The image $\mathrm{im}(\bar\rho)\subseteq I'_{t+1}$ is a sub-copresheaf of transported evidence. At each object $A'$, the pointwise residual set
    \[
    \mathcal{R}(A') = I'_{t+1}(A') \,\setminus\, \mathrm{im}(\bar\rho_{A'})
    \]
    records artifacts at that type not obtained by functorial transport of old evidence. The categorical object is the inclusion $\mathrm{im}(\bar\rho)\hookrightarrow I'_{t+1}$; the complements $\mathcal{R}(A')$ are an objectwise engineering diagnostic and need not themselves form a sub-copresheaf.
    \item When a relative description-length functional $L_{b'}(-\mid-)$ is present, the \textbf{discovery cost} of the move is
    \[
    L_{b'}\big(I'_{t+1}\mid \mathrm{im}(\bar\rho)\big),
    \]
    the code length of the post-transition state given transported evidence. If $L_{b'}$ is additive on such inclusions, this can be written as $L_{b'}(I'_{t+1})-L_{b'}(\mathrm{im}(\bar\rho))$. In particular, any description-length functional assigning positive conditional length to nonempty forced residuals gives a strictly positive local discovery cost at such a type.
\end{enumerate}
\end{proposition}

\begin{proof}
The adjunction $\Lan_u\dashv u^*$ gives a natural bijection
\[
\mathrm{Nat}(\Lan_u I_t,I'_{t+1})\cong \mathrm{Nat}(I_t,u^*I'_{t+1}),
\]
so $\rho$ determines the unique map $\bar\rho$. Statement~(i) is the pointwise formula~(\ref{eq:lan_formula}) for an empty index category: a colimit over the empty diagram in $\Set$ is the initial object $\varnothing$. Hence the component of $\bar\rho$ at $A'$ has empty domain and empty image, so any accepted element of $I'_{t+1}(A')$ lies outside transport. For (ii), the pointwise image of a natural transformation between Set-valued functors is a subfunctor: naturality sends elements in $\mathrm{im}(\bar\rho_{A'})$ to elements in $\mathrm{im}(\bar\rho_{B'})$ along every morphism $A'\to B'$. Therefore $\mathrm{im}(\bar\rho)$ is the transported-evidence substate, and the pointwise complement records artifacts not in that transported image. Statement~(iii) is the definition of the relative description-length functional on the inclusion $\mathrm{im}(\bar\rho)\hookrightarrow I'_{t+1}$, with strict positivity following from the stated positivity assumption on forced residuals.
\end{proof}

\begin{proposition}[Refinements lift to realized provenance]
\label{prop:elements_functoriality}
Let $b$ be fixed and let $\alpha:I\to J$ be a refinement morphism in $\K_b$, represented in the strict presentation by a componentwise injective natural transformation. Then $\alpha$ induces a faithful functor
\[
\int\alpha:\int_{\Sschema_b} I\longrightarrow \int_{\Sschema_b} J
\]
defined by $(A,x)\mapsto (A,\alpha_A(x))$ on objects and by sending a realized operation $f:(A,x)\to(B,y)$ to the same schema morphism $f:(A,\alpha_A(x))\to(B,\alpha_B(y))$. The assignment $\alpha\mapsto\int\alpha$ preserves identities and composition of refinements.
\end{proposition}

\begin{proof}
If $f:(A,x)\to(B,y)$ is a morphism in $\int_{\Sschema_b}I$, then $I(f)(x)=y$. Naturality of $\alpha$ gives
\[
J(f)(\alpha_A(x))=\alpha_B(I(f)(x))=\alpha_B(y),
\]
so the same schema morphism $f$ is a morphism in $\int_{\Sschema_b}J$. Thus $\int\alpha$ is a functor. It is faithful because it does not change the underlying schema morphism on any hom-set. Identity refinements act as identity functors, and for refinements $\alpha:I\to J$ and $\beta:J\to L$, the componentwise equality $(\beta\circ\alpha)_A=\beta_A\circ\alpha_A$ gives $\int(\beta\circ\alpha)=\int\beta\circ\int\alpha$.
\end{proof}

\begin{remark}
Proposition~\ref{prop:elements_functoriality} is the formal audit rule behind the artifact-ledger language. If a refinement adds, annotates, or supersedes artifacts without identifying previously distinct accepted artifacts, then the realized provenance graph of the predecessor embeds faithfully into the realized provenance graph of the successor. Older lineage queries remain meaningful after accepted refinements.
\end{remark}

\begin{proposition}[Verified regime transitions compose]
\label{prop:transition_composition}
Let $(u,\rho)$ be a verified regime transition from $b$ to $b'$ relative to $(I_t,I'_{t+1})$, and let $(v,\sigma)$ be a verified regime transition from $b'$ to $b''$ relative to $(I'_{t+1},I''_{t+2})$. Assume that gates are hereditary on declared transported substates: whenever a transported state is accepted, any transported substate selected by a preservation map is also accepted. Then
\[
(v\circ u,\tau),\qquad
\tau=(u^*\sigma)\circ\rho:I_t\longrightarrow (v\circ u)^*I''_{t+2},
\]
is a verified regime transition from $b$ to $b''$ relative to $(I_t,I''_{t+2})$. Moreover, under the canonical isomorphism $\Lan_{v\circ u}\cong \Lan_v\Lan_u$, the comparison map $\bar\tau$ factors as
\[
\Lan_{v\circ u}I_t \cong \Lan_v\Lan_u I_t
\xrightarrow{\Lan_v(\bar\rho)}
\Lan_v I'_{t+1}
\xrightarrow{\bar\sigma}
I''_{t+2}.
\]
\end{proposition}

\begin{proof}
The composite $v\circ u$ preserves morphism faithfulness and object injectivity except at type splits, quotients, or recodings explicitly recorded by the two transitions. Since restriction along a composite satisfies $(v\circ u)^*=u^*v^*$, the map $\tau=(u^*\sigma)\circ\rho$ is natural, componentwise injective, and provenance-preserving because both factors are. If $\V_b(I_t)=1$, the first transition accepts the transported image in $b'$, and by the acceptance condition on the new state $\V_{b'}(I'_{t+1})=1$. The gate condition of the second transition then gives $\V_{b''}(\mathrm{im}(\bar\sigma))=1$. Since the displayed factorization makes $\mathrm{im}(\bar\tau)\subseteq\mathrm{im}(\bar\sigma)$ a transported substate selected by a preservation map, heredity gives $\V_{b''}(\mathrm{im}(\bar\tau))=1$. The factorization of $\bar\tau$ follows from uniqueness of left adjoints to restriction: $\Lan_{v\circ u}$ and $\Lan_v\Lan_u$ are both left adjoint to $u^*v^*$, and the displayed composite is the adjoint transpose of $\tau$.
\end{proof}

\begin{remark}
This proposition justifies reading a multi-stage investigation as one audited discovery move when every intermediate transition records a preservation map and clears its gate. In the Builder/Breaker case, the four accepted outer iterations can therefore be viewed both locally, transition by transition, and globally, as a composite map transporting the initial evidence into the final symbolic regime while accumulating residual content.
\end{remark}

\begin{proposition}[Declared old evidence remains inspectable]
\label{prop:evidence_conservativity}
Let $(u,\rho)$ be a verified regime transition from $b$ to $b'$ relative to $(I_t,I'_{t+1})$. For every old type $A$ and artifact $x\in I_t(A)$, the component $\rho_A$ gives a unique preserved artifact $\rho_A(x)\in I'_{t+1}(uA)$ with the same declared identity and provenance metadata, up to the equivalence or tolerance specified by the transition. If $J\hookrightarrow I_t$ is an old accepted evidence substate and the gate-compatibility condition declares its transported image accepted in $b'$, then all artifacts of $J$ remain accepted as transported evidence in $I'_{t+1}$.
\end{proposition}

\begin{proof}
Componentwise injectivity and provenance preservation of $\rho$ are part of Definition~\ref{def:transition}. Therefore distinct old artifacts remain distinct after applying $\rho$, and their recorded lineage remains inspectable in the new state. For an accepted evidence substate $J\hookrightarrow I_t$, restricting $\rho$ gives a preservation map from $J$ into $u^*I'_{t+1}$, equivalently a comparison map from its transported image into $I'_{t+1}$. The gate-compatibility hypothesis then says precisely that this transported image is accepted in $b'$.
\end{proof}

\begin{remark}
The proposition is deliberately about preserved evidence, not about every old model commitment. Discovery systems may explicitly retract, compress, or supersede symbolic model terms when the new gate rejects them on accumulated evidence. What is prohibited is silent deletion of the old evidential record. In the Builder/Breaker audit, the additive-model parameters retracted at the $2\to3$ transition are model commitments recorded as superseded; the protein evidence and its provenance remain available for the accepted multiplicative law.
\end{remark}

\begin{proposition}[Multicategorical Kan obstruction]
\label{prop:multicat_kan}
Assume a regime is presented by a small typed multicategory, or equivalently a colored-operadic schema, whose multimorphisms
\[
\phi:(A_1,\ldots,A_k)\longrightarrow B
\]
declare admissible multi-input scientific operations. Assume artifact states are Set-valued algebras on these schemas and that restriction along a multifunctor $u:M_b\to M_{b'}$ admits a left adjoint $\Lan^{\mathrm m}_u$ computed by the usual operadic comma-colimit construction \citep{leinster_higher_operads_2004}. For a verified multicategorical transition with preservation map $\rho:I_t\to u^*I'_{t+1}$ and adjoint comparison map
\[
\bar\rho:\Lan^{\mathrm m}_u I_t\longrightarrow I'_{t+1},
\]
the following hold.
\begin{enumerate}[leftmargin=1.7em,topsep=2pt,itemsep=2pt]
    \item If the operadic comma category of operations from transported old types into a new type $B$ is empty in every arity, then $(\Lan^{\mathrm m}_u I_t)(B)=\varnothing$. Any accepted artifact at $B$ is forced residual content, exactly as in Proposition~\ref{prop:kan_obstruction}.
    \item If $B$ has no unary generating morphism from the old image but receives a new $k$-ary multimorphism from transported or generator-reachable inputs, then $B$ is not isolated in the multicategorical regime. Its transported part is generated by formal composites of transported input artifacts through that multimorphism, while the newly admitted multimorphism itself is residual schema content.
    \item If the description-length functional separates artifact content from operation-registry content, then the relative discovery cost splits into an artifact residual term and an operation residual term; without additivity, the same statement holds as a subadditive upper bound.
\end{enumerate}
\end{proposition}

\begin{proof}
The proof is the multicategorical analogue of Proposition~\ref{prop:kan_obstruction}. By assumption, $\Lan^{\mathrm m}_u$ is left adjoint to restriction and is computed pointwise by a colimit over the operadic comma category whose objects are multimorphisms from tuples of transported old types to the target type. If this indexing category is empty in every arity, the colimit in $\Set$ is $\varnothing$, forcing any accepted element of $I'_{t+1}(B)$ to lie outside transport. If a $k$-ary multimorphism into $B$ exists, the same comma-colimit has nonempty indexing data in arity $k$ and therefore contains the formal composites obtained by applying that operation to transported input artifacts; applying $\bar\rho$ evaluates those formal composites in the verified new state. Finally, a relative code that is additive over artifact fibers and operation-registry entries gives the stated decomposition of residual cost, while subadditivity gives the upper bound in the nonadditive case.
\end{proof}

\begin{remark}
For the Builder/Breaker $2\to3$ transition, the ordinary generator-level audit sees no unary morphism from the old schema into \texttt{ModeConditionedCompliance}. The multicategorical view sees the binary operation
\[
\texttt{LogNormCompliance}\times\texttt{ReLUModeAmpl}
\longrightarrow \texttt{ModeConditionedCompliance},
\]
whose inputs are themselves generator-reachable transforms of old physics-derived quantities. Thus Table~\ref{tab:kan_audit} and Fig.~\ref{fig:kan_audit_builder_breaker} are the unary-shadow audit of a native multicategorical statement: the target feature is composite-reachable, and the product operation is residual schema content.
\end{remark}

\subsection{CategoryScienceClaw mechanics figure-generation details}
\label{sec:categoryscienceclaw_methods}

The CategoryScienceClaw mechanics and materials examples were generated as replayable export trees. Each run contains an investigation JSON file, a discovery report, a categorical discovery graph, and, when needed, deterministic computational inputs. \hlyellow{The main manuscript reports the HEA fatigue-crack-growth run as the worked CategoryScienceClaw application and includes its generated PDF figure. The fiber-network model-comparison run is reported in the Supplementary Information as a supporting audit case.}

\hlyellow{The visualized computations use ScienceClaw skills and deterministic local analyses: HEA fatigue-crack-growth parsing and threshold--slope diagnostics for the main CategoryScienceClaw case, and orientation-tensor plus stress--strain regression for the supporting fiber-network run. Evidence origin is retained in the investigation JSON, so imported structures, computational surrogates, and computed input tables remain distinguishable. This provenance labeling supports reproducible mechanics reasoning and categorical audit without turning the example cases into empirical validation claims.}

The categorical discovery graph records candidate models, accepted and rejected status, gates, stress tests, and report/figure artifacts. Thus a plotted panel is traceable back to input artifacts and morphisms rather than being a standalone graphics product. The manuscript figures use the generated PDF outputs directly, preserving the provenance reported by CategoryScienceClaw.

\subsection{Case-study implementation details}

The Builder/Breaker protein-mechanics case uses Gaussian Network Model features derived from PDB chains, symbolic DAG search over physically interpretable factors, and paired MDL comparisons. The Breaker selects staged protein evidence intended to expose failure modes; the Builder proposes edits such as adding factors, removing factors, swapping terms, and thresholding. \hlyellow{The figures in this manuscript summarize the outer four-stage trajectory as well as the inner hill-climb, gate-anatomy, and feature-lifecycle diagnostics. As reported in} \cite{buehler_break_world_2026}, \hlyellow{the algorithm uses GPT-5.5 (OpenAI) as the underlying large language model for both Builder/Breaker agents, with extended reasoning enabled (\texttt{reasoning\_effort=high}). Exact bitwise replay of the proposal generator is not guaranteed because those proposals depend on closed, stochastic model calls and service-side model state. The deterministic parts of the analysis--PDB-derived GNM feature construction, regression/refitting, MDL scoring of recorded candidate states, Kan-transport audit, and figure-generation summaries--are separated from those model-call-dependent proposal steps in the released artifacts and companion repository.
}

The Kan-transport audit in Fig.~\ref{fig:kan_audit_builder_breaker} and Table~\ref{tab:kan_audit} was computed from the accepted world-model records of the same run. For each outer iteration $t$, a finite schema $\Sschema_t$ was constructed whose objects are the typed physics artifacts, observables, symbolic intermediate features, B-factor target, and singleton parameter objects present in the accepted model. The artifact-state fiber size for residue-level objects was taken from the accumulated evidence count recorded in the run, giving $122$, $263$, $691$, and $1171$ residues at iterations $0$ through $3$; chain-level physics objects used the corresponding accumulated chain counts. 
Note that the reported $R^2$ values in Fig.~\ref{fig:case_dagmdl} are stagewise descriptive summaries on these accumulated
evidence sets. They are not computed on a fixed held-out benchmark shared across all four
outer iterations. For this reason, changes in $R^2$ across outer iterations are not used as the
acceptance rule for discovery moves. Acceptance is determined by paired MDL comparison:
at each proposed revision, the incumbent and candidate symbolic DAGs are refit and scored
on the same accumulated evidence available at that stage.

For each transition $\Sschema_t\to\Sschema_{t+1}$, shared objects were transported by identity. New objects were classified in two passes. First, the generator-level comma diagnostic recorded whether the new object received an immediate unary generating morphism from a shared old object. Second, a composite-reachability pass took the transitive closure through all new morphisms, including multi-input product morphisms. Thus an object such as \texttt{ModeConditionedCompliance} is not generator-reachable from the old schema, but it is composite-reachable through newly admitted morphisms
\[
\texttt{Compliance}\to\texttt{LogNormCompliance},\qquad
\texttt{NormModeAmpl}\to\texttt{ReLUModeAmpl},
\]
followed by the product operation defining \texttt{ModeConditionedCompliance}. This two-level audit avoids treating the empty generator-level comma category as a claim of absolute isolation in the full free category. The residual fiber counts in the audit should therefore be read objectwise: for shared and generator-reachable residue-level types, the $480$-residue residual in the final transition is the newly revealed evidence slice; for composite-reachable types, the essential residual is the new operation and feature type that make the old evidence composable in a new way. MDL break gains are reported as paired acceptance gains, not as the relative code length $L_{b'}(I'_{t+1}\mid\mathrm{im}(\bar\rho))$ itself.

\hlyellow{CategoryScienceClaw is treated as ScienceClaw equipped with a typed categorical and proof-carrying layer rather than only as a numeric model-selection run. The implementation provides a ScienceClaw skill registry and adapter, controlled artifact-type metadata, an evolving artifact ledger, parent-linked immutable lineage, open needs, pressure-based reactions, mutation of active status, public discourse artifacts with verifier signals, proof certificates, and domain-specific accepted/rejected model records. Operationally, ScienceClaw supplies the executable skills, routing substrate, artifact production, and discourse-facing workflow; CategoryScienceClaw supplies the categorical contracts and proof objects that make those executions auditable. The formal reading used in the paper is therefore conservative: the code realizes a typed artifact family and an acyclic provenance hypergraph, whose path-category shadow gives the category-of-elements-style structure discussed in the Results. We use GPT-5.2 (OpenAI) as backbone LLM for the CategoryScienceClaw studies.}

% =====================================================================
\section*{Supplementary information}
% =====================================================================

\hlyellow{Supplementary Information provides additional details related to the CategoryScienceClaw case studies, including the HEA fatigue replay files, the supporting fiber-network mechanics case, and terminology scaffolding.}

\section*{Author contributions}

M.J.B. conceived the technical focus, project goals and investigation scope and wrote the initial draft of the manuscript. 
F.Y.W. designed and developed the CategoryScienceClaw system, including the agent framework, skill library, artifact system, and multi-agent coordination; ran and analyzed case studies. M.J.B. developed the Builder-Breaker model and conducted the associated analysis. All authors participated in the development of the study, analysis of results, and interpretation, and the writing of the paper.  

\section*{Use of generative AI}

During manuscript preparation, an LLM was used to assist with language editing and code editing for benchmark-generation scripts. The authors reviewed and edited the output, verified the scientific content, and are responsible for the final manuscript.

% =====================================================================
\section*{Code and data availability}
% =====================================================================

Code, generated artifacts, logs, and figures for the Builder/Breaker protein-mechanics case are available at \url{https://github.com/lamm-mit/BreakingTheWorld}. \hlyellow{The released materials support replay of the deterministic analysis steps, including recorded protein-mechanics feature tables, refits, MDL comparisons on recorded states, and post-hoc categorical audits; exact regeneration of every raw LLM proposal is not claimed.} The ScienceClaw framework is available at \url{https://github.com/lamm-mit/scienceclaw} and the CategoryScienceClaw mechanics branch at \url{https://github.com/lamm-mit/scienceclaw/tree/categoryscienceclaw-mechanics}. The Infinite discourse platform is available at \url{https://github.com/lamm-mit/infinite}. \hlyellow{The CategoryScienceClaw HEA fatigue-crack-growth investigation is fully reproducible and available at} \url{https://github.com/lamm-mit/scienceclaw/tree/categoryscienceclaw-mechanics/categoryscienceclaw/HEA} \hlyellow{and includes HEA input data from Materials Cloud and typed audit artifacts (input database, rejected threshold-only model, accepted two-axis model, science claim), the categorical discovery graph with full provenance.}

% =====================================================================
\section*{Acknowledgments}
% =====================================================================

Support from MGAIC, ONR, ARO, AFOSR, NIH, DSO, NSF, and additional sources is gratefully acknowledged. Part of this work was supported by the U.S. Department of Energy, Office of Science, Office of Advanced Scientific Computing Research and Office of Basic Energy Sciences, Scientific Discovery through Advanced Computing (SciDAC) program under the FORUM-AI project.

\section*{Competing interests}
The authors declare that they have no competing interests.

% =====================================================================
% Bibliography
% =====================================================================

%\bibliographystyle{unsrtnat}
\bibliographystyle{naturemag}

\bibliography{refs}

\clearpage % Ensures the SI starts on a completely new page



\section*{Supplementary material}

Supplementary Information

# Self-Revising Discovery Systems for Science: A Categorical Framework for Agentic Artificial Intelligence

Fiona Y. Wang[1,2] Markus J. Buehler[2,3,4*]
[1]Laboratory for Atomistic and Molecular Mechanics, MIT
[2]Department of Biological Engineering, MIT
[3]Departments of Civil and Environmental Engineering and Mechanical Engineering,
[4]Center for Computational Science and Engineering, Schwarzman College of Computing, MIT
Cambridge, MA 02139, USA; [*]Corresponding author: mbuehler@mit.edu

This Supplementary Information file collects terminology scaffolding and a supporting mechanics case. CategoryScienceClaw is the categorical/proof-carrying layer over ScienceClaw: it preserves the ScienceClaw skill registry, artifact lineage, pressure coordination, workflow mutation, and public discourse, while making typed objects, morphisms, needs, certificates, and audits explicit. The main text uses the high-entropy-alloy (HEA) fatigue-crack-growth investigation as the primary worked CategoryScienceClaw case. The fiber-network mechanics run is retained here as a supporting model-comparison audit, showing the same typed-provenance skeleton on a second, independent mechanics example. SI shows reproducibility and provenance such that this case preserves candidate models, accepted and rejected alternatives, gates, stress tests, categorical graph summaries, and figure references.

## Terminology Guide

The table below explains the technical vocabulary used in the main text, including terms in mechanics, statistics, machine learning, materials science, and software provenance objects.

**Table S1:** Term-by-term guide to category-theory, mechanics, statistical/model-selection, scientific-machine-learning, provenance, and evidence-scope vocabulary used in the paper.

| Term | Plain-language meaning | Role in this paper |
|---|---|---|
| Schema category $\mathcal{S}$ | The list of allowed evidence types and allowed transformations between them. It specifies the grammar of a discovery regime. | A mechanics schema says which objects can be represented, such as fatigue threshold, Paris slope, phase class, orientation tensors, and stress–strain fits. A regime transition changes this schema rather than only changing parameter values. |
| Object $A \in \mathcal{S}$ | A typed scientific record: a dataset, descriptor, model, diagnostic, or claim with a declared kind. | Examples include a threshold descriptor, a Paris-slope object, a phase-class label, an orientation tensor, and a damage-tolerance claim. |
| Morphism $f : A \to B$ | An admissible operation that turns one typed record into another while preserving its provenance. | Examples include computing threshold and Paris slope from a fatigue-crack-growth record, fitting a stress–strain model, or evaluating a damage-tolerance gate. |
| Copresheaf $X : \mathcal{S} \to \mathbf{Set}$ | A bookkeeping assignment from each schema object to the concrete records currently available for that object type. | In this paper it is the committed-state semantics: $X(A)$ is the set of committed artifacts of type $A$ at a given point in the run. |
| Category of elements $\int_{\mathcal{S}} X$ | The realized provenance graph obtained by listing each concrete artifact together with the schema object it instantiates. | It is the audit graph: nodes are concrete records and edges are the operations that generated them from parent records. |
| Committed state | The current accepted typed record of what the system is willing to stand behind. | Rejected alternatives can remain in the audit trail, but the committed state identifies the accepted descriptors, models, gates, and claims used for the paper's scientific interpretation. |
| Endofunctor | A structure-preserving update that maps a schema or workflow back into the same overall kind of schema or workflow. | It represents fixed-regime revision: the system can refine, prune, or recompose artifacts without changing what counts as an admissible type of evidence. |

| Term | Plain-language meaning | Role in this paper |
|---|---|---|
| Left Kan extension | The most systematic way to carry old records forward when the schema is enlarged. | It formalizes transport across a regime transition, such as carrying old threshold and phase-class evidence into the two-axis HEA schema while adding the Paris-slope object that was previously unavailable. |
| Transport | The part of a new interpretation that can be inherited from the old schema without adding genuinely new kinds of evidence. | In the HEA case, composition, phase class, testing condition, and $\Delta K_{th}$ records can transport from the old threshold-only regime. |
| Residual content | The part of the new regime that cannot be recovered by transporting old records alone. | Paris-slope stability, the threshold–slope map, and the high-threshold/high-slope risk class are residual content in the HEA case. |
| Agentic discovery | A workflow in which software agents propose, test, reject, revise, and report scientific artifacts rather than only executing a fixed script. | The paper's claim is not that agency replaces mechanics validation, but that agentic workflows need typed records of what was proposed, accepted, rejected, and carried forward. |
| Scientific-machine-learning model | A fitted or learned relation constrained by scientific variables, data, and diagnostics rather than an unconstrained black-box predictor. | The case studies use simple fitted models as audit objects so the categorical machinery can be inspected; they are not presented as state-of-the-art predictive models. |
| Model-comparison case | A study where the main output is an auditable choice among candidate descriptions. | The fiber-network case compares anisotropy-aware and null descriptions; the categorical contribution is the typed preservation of the accepted model, rejected model, and gate. |
| Regime-transition case | A study where the relevant variables change, not just the parameters of a fixed model. | The HEA fatigue case moves from a threshold-only descriptor to a two-axis threshold/Paris-slope descriptor; this is why transport and residual content matter. |
| Independent validation | An external check against held-out experiments, independent simulations, or literature benchmarks not used to build the claim. | The present cases demonstrate reproducible categorical auditing and mechanics-facing regime change. They should be read as validation of the framework's representation and replay behavior, not as independent experimental discovery claims. |
| Gate | A typed acceptance test that records why one model or claim is accepted over another. | AIC, MDL, and the HEA damage-tolerance threshold/slope criterion are gates that make acceptance criteria explicit rather than implicit. |
| AIC/MDL gate | A statistical or information-theoretic rule balancing fit quality against model complexity. | For example, AIC uses $2k - 2\log L$ and MDL asks whether shorter total code length justifies a proposed model. These gates support model comparison, not absolute proof of truth. |
| Rejected alternative | A model or descriptor that was considered, typed, and preserved even though it did not pass the gate. | The threshold-only HEA screening rule and the isotropic fiber-count descriptor are retained as audit objects so each accepted model has a documented contrast class. |
| Stress test | A targeted perturbation or robustness check applied after model selection. | Examples include orientation perturbations in the fiber-network case and the cross-phase threshold/slope ratio check in the HEA case. |
| Deterministic replay | The files and metadata needed to rerun the deterministic part of a case and reproduce its typed artifacts. | The SI points to CSV and JSON inputs plus audit records so the orientation-tensor, threshold, Paris-slope, gate, and claim artifacts can be checked end to end. |
| Constitutive/material model | The assumed relation between material properties, geometry, loading, and response. Typical inputs include elastic modulus, Poisson ratio, stiffness, or strength parameters. | In this manuscript, such quantities should be read as declared computational assumptions unless an external experiment or benchmark is explicitly cited. |
| Structure–property relation | A materials-science claim connecting geometry, microstructure, or defects to an effective property or failure response. | The paper uses this idea cautiously: the case studies show how such claims would be typed, gated, and replayed, but they do not establish a new experimentally validated structure–property law. |
| Typed skill | A software operation with declared input and output types. | CategoryScienceClaw uses typed skills to ensure that a generated artifact can be checked against the schema object it claims to instantiate. |
| Proof certificate | A compact machine-readable record that a typed operation ran with the expected parents, outputs, and checks. | Certificates connect implementation artifacts to the categorical audit vocabulary, making the workflow inspectable rather than only narrative. |
| Categorical discovery graph | The exported graph of typed artifacts, generating operations, open needs, gates, and claims. | It is the implementation-level form of the category-of-elements audit used by the main text and the SI mechanics cases. |

| Term | Plain-language meaning | Role in this paper |
|---|---|---|
| Mechanics investigation sidecar | A JSON summary that stores the case-specific mechanics interpretation beside the lower-level artifacts. | It records accepted and rejected models, quantitative diagnostics, stress tests, regime-transition statements, limitations, and unsupported-claim guards. |

## Supporting Mechanics Cases

| Run | Accepted model | Rejected model | Gate, evidence origin, and figure reference |
|---|---|---|---|
| 7T10 structure-contact mechanics | linear force–extension model | mean-force null model | AIC gate applied to a deposited structure/contact graph and one imported coarse-grained force–extension trace. Figure S1. |
| Mechanobiology force paths | four-feature graph-conditioned regression | adhesion-only traction model | Ablation gate and strongest-path-removal test on a deterministic synthetic 12-path graph. Figure S2. |
| Membrane biophysics | Helfrich-style curvature-energy proxy | curvature-only descriptor | Regime-transition gate and modulus sensitivity on a deterministic synthetic $7 \times 7$ field. Figure S3. |
| Fiber-network mechanics | orientation-tensor stiffness surrogate | isotropic fiber-count descriptor | AIC gate and orientation perturbation on a deterministic synthetic 12-fiber network. Figure S4. |

**Table S2:** Supporting CategoryScienceClaw mechanics cases. Each case preserves the accepted and rejected models, the gate or stress test separating them, the origin of its evidence, and the corresponding SI figure reference.

Like the main-text HEA case, these four supporting cases follow the same typed audit skeleton: a candidate model set, an accepted model, a rejected model, a model-selection gate or stress test, and a scoped claim. The rejected alternative is preserved as a typed provenance artifact rather than a deleted failure. The evidence sources differ, however. The fiber-network, mechanobiology, and membrane cases use deterministic synthetic computational inputs. The 7T10 case instead combines a deposited PDB structure-derived contact graph with one imported coarse-grained computational force–extension trace. None of the four cases constitutes new experimental validation.

## 7T10 Structure-Contact Mechanics

The 7T10 case asks whether the peptide–receptor contact graph derived from the deposited PDB 7T10 structure, together with one imported coarse-grained force–extension trace from 0 to 2.4 nm, supports a contact-localized tensile-mechanics description. It compares a linear force–extension model with a mean-force null model using an AIC gate. The audit reports contact-hotspot positions $[8, 9, 6, 7, 1, 11]$, fitted stiffness 253.068938 $\mathrm{pN\,nm^{-1}}$, peak force 766.98959 pN, and $R^2 = 0.942723$. The structural input is not synthetic, but the force trace is a single computational surrogate rather than measured mechanics, a replicated simulation ensemble, or a validated atomistic uncertainty estimate. The result therefore supports the internal computational contrast shown in Figure S1, not an experimentally established stiffness for 7T10.

## Mechanobiology Force Paths

The mechanobiology case uses a deterministic synthetic table of 12 paths and a synthetic adhesion–cytoskeleton adjacency graph. Each path has an adhesion score, cytoskeletal-coupling score, path length, displacement, and traction proxy. The accepted four-feature graph-conditioned regression is compared with an adhesion-only traction model; strongest-path removal tests whether the result is dominated by one high-load route. The audit reports mean load-path score 4.421814 $\mathrm{Pa}\,\mu\mathrm{m}^{-1}$, maximum traction 72.886 Pa on path 12, and $R^2 = 0.398862$ for the adhesion-only fit. Figure S2 therefore demonstrates typed model comparison and load-routing sensitivity on controlled computational inputs, not traction-force microscopy, cell-specific calibration, or causal biological inference.

## Membrane Biophysics

The membrane case uses a deterministic synthetic $7 \times 7$ curvature field and an illustrative material-model sidecar with bending modulus 20 $k_B T$. It asks whether a Helfrich-style quadratic curvature-energy proxy adds material mechanics beyond a curvature-only geometric descriptor. The diagnostic varies the bending modulus over $0.5\times$, $1\times$, and $2\times$ its

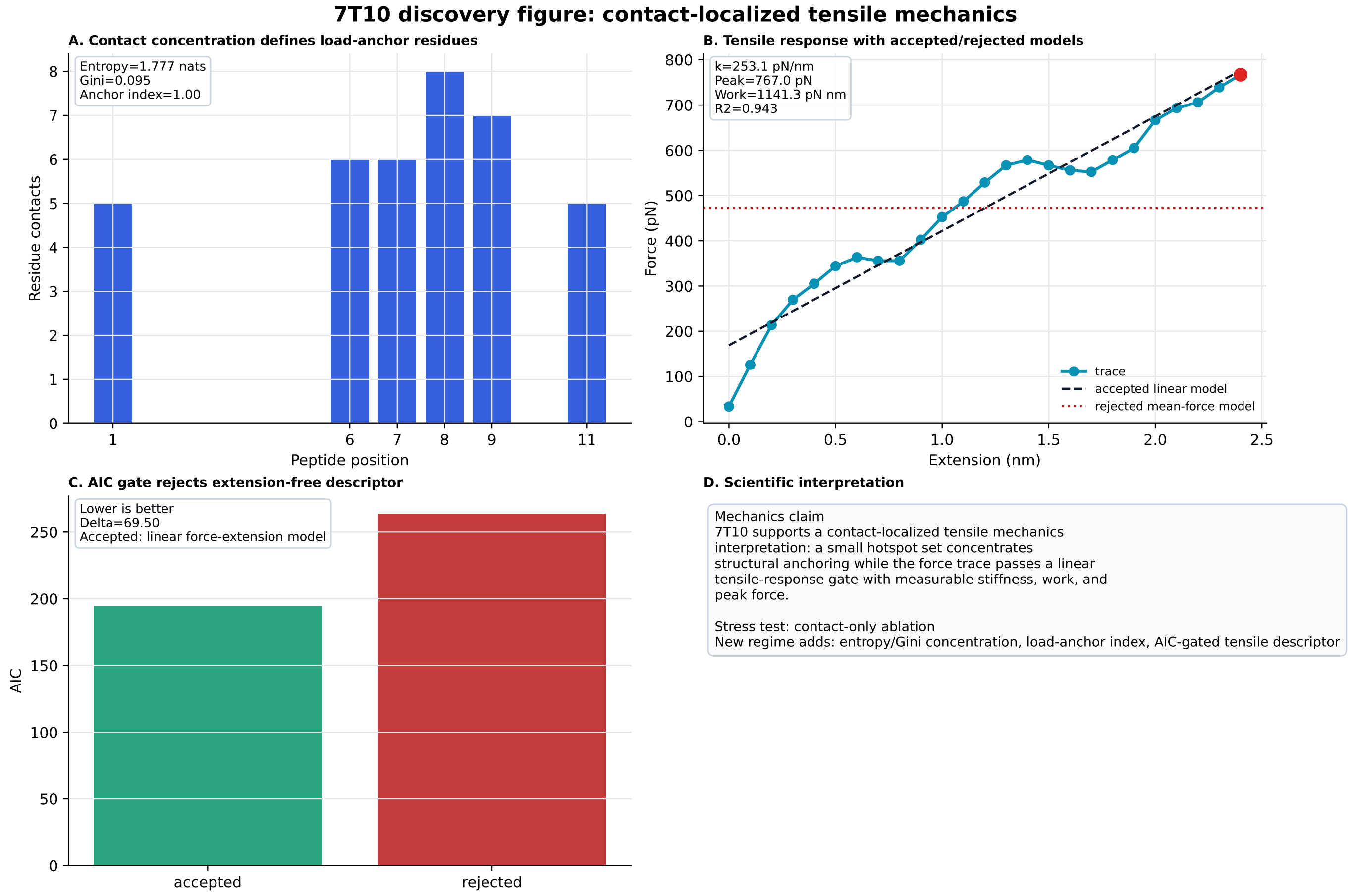


**Figure S1:** 7T10 structure-contact mechanics audit. A deposited structure-derived peptide–receptor contact graph and one imported coarse-grained force–extension surrogate are used to compare the accepted linear tensile model with a mean-force null descriptor under an AIC gate. The figure records contact-hotspot concentration, the force–extension fit, and the resulting scoped mechanics claim; it does not represent measured force spectroscopy or a validated atomistic ensemble.

nominal value; the audit reports RMS curvature 0.15471241 $\mu\text{m}^{-1}$, mean energy proxy 0.23935931 $k_BT\,\mu\text{m}^{-2}$, and total grid energy proxy 11.72860601 $k_BT$. Figure S3 illustrates the transition from geometry to an energy-bearing model, but the field is not a measured membrane reconstruction and the parameters are not calibrated to lipid composition or fluctuation data.

## Fiber-Network Mechanics

The fiber-network run uses a deterministic synthetic 12-fiber table, with orientations from 12° to 93°, an 11-point uniaxial stress–strain table over strain 0.00–0.10, and a computational-model sidecar. It asks whether the response is better represented by a scalar fiber-count descriptor or by an orientation-tensor anisotropic stiffness surrogate. The input is not tissue segmentation, biological measurement, or a finite-element boundary-value solution. The accepted model, rejected alternative, AIC gate, and orientation-perturbation stress test make the computational decision auditable.

The released materials allow the orientation tensor, principal axis, linear stress–strain fit, AIC comparison, perturbation stress test, and accepted/rejected model records to be checked end to end. Figure S4 is the original four-panel fiber-network audit: it shows the dominant load-bearing axis, retained stiffness scale, AIC model-selection gate, and orientation-perturbation stress test.

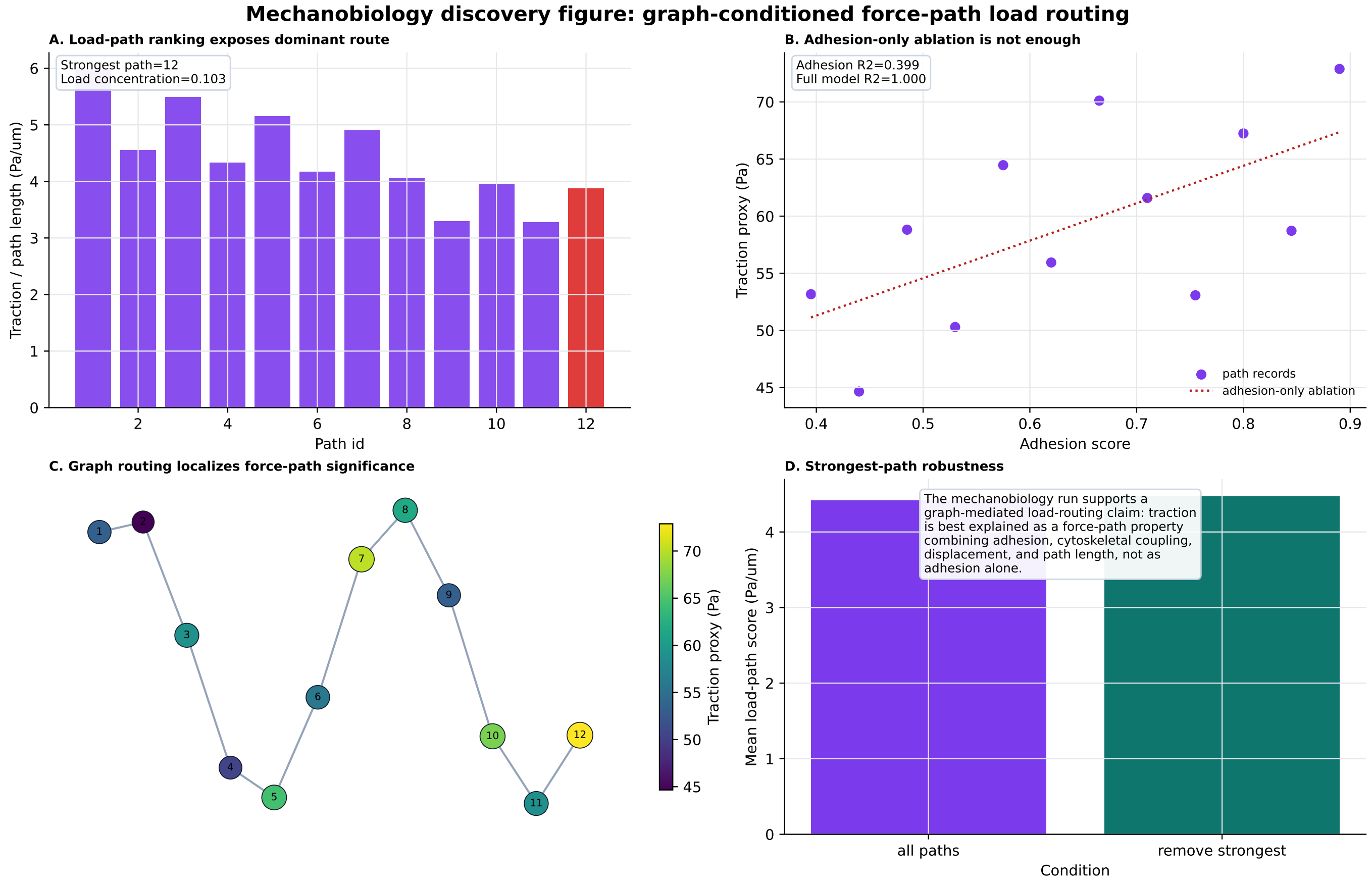


**Figure S2:** Mechanobiology force-path audit on a deterministic synthetic 12-path adhesion–cytoskeleton graph. The figure compares the accepted graph-conditioned four-feature regression with the rejected adhesion-only model and records the strongest-path-removal stress test. The traction values are computational proxies, not measured cellular forces.

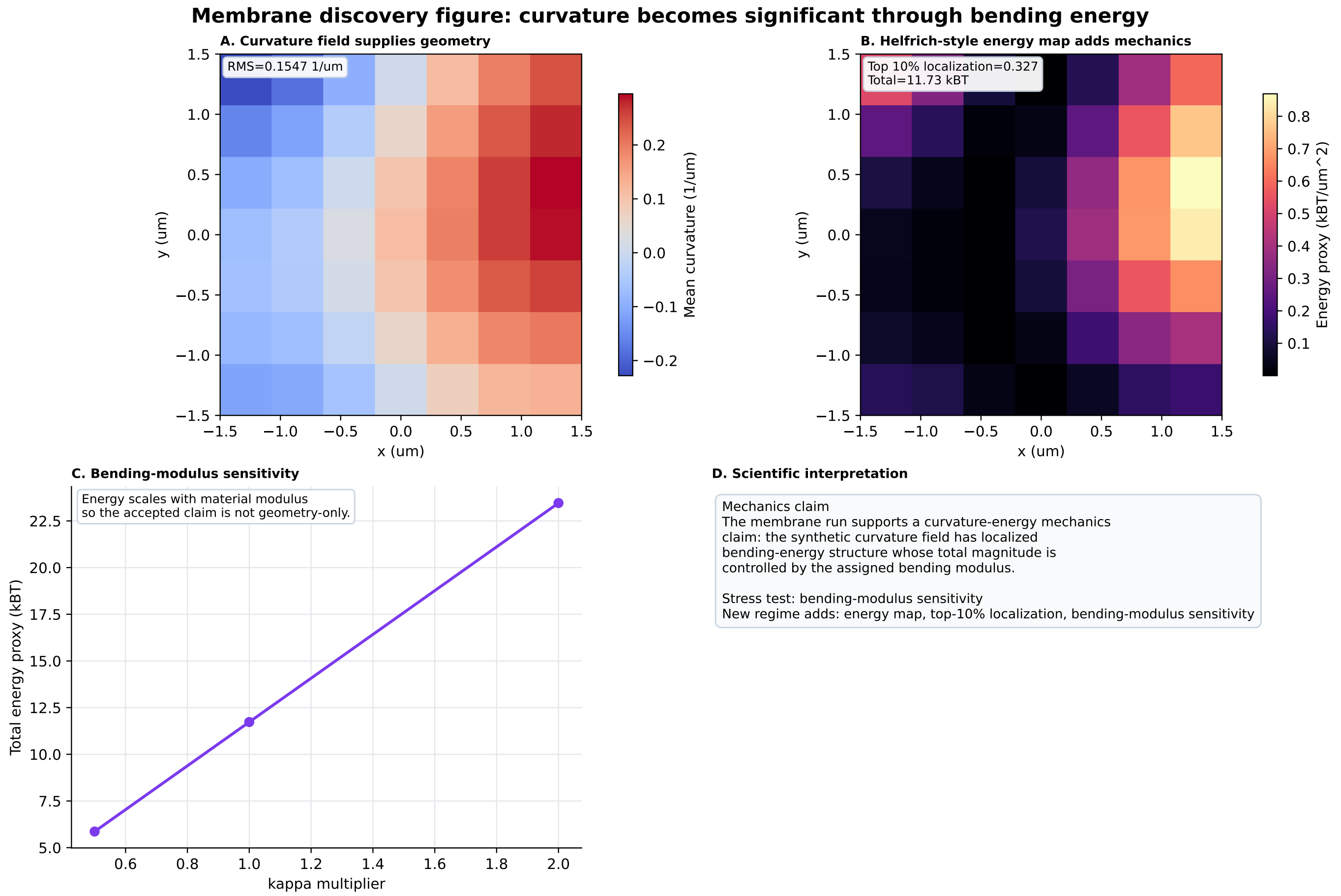


**Figure S3:** Membrane-biophysics audit on a deterministic synthetic $7 \times 7$ curvature field. The accepted Helfrich-style curvature-energy proxy is compared with a curvature-only descriptor, and bending-modulus sensitivity tests the added material dependence. The figure is an illustrative computational case, not a calibrated membrane measurement.

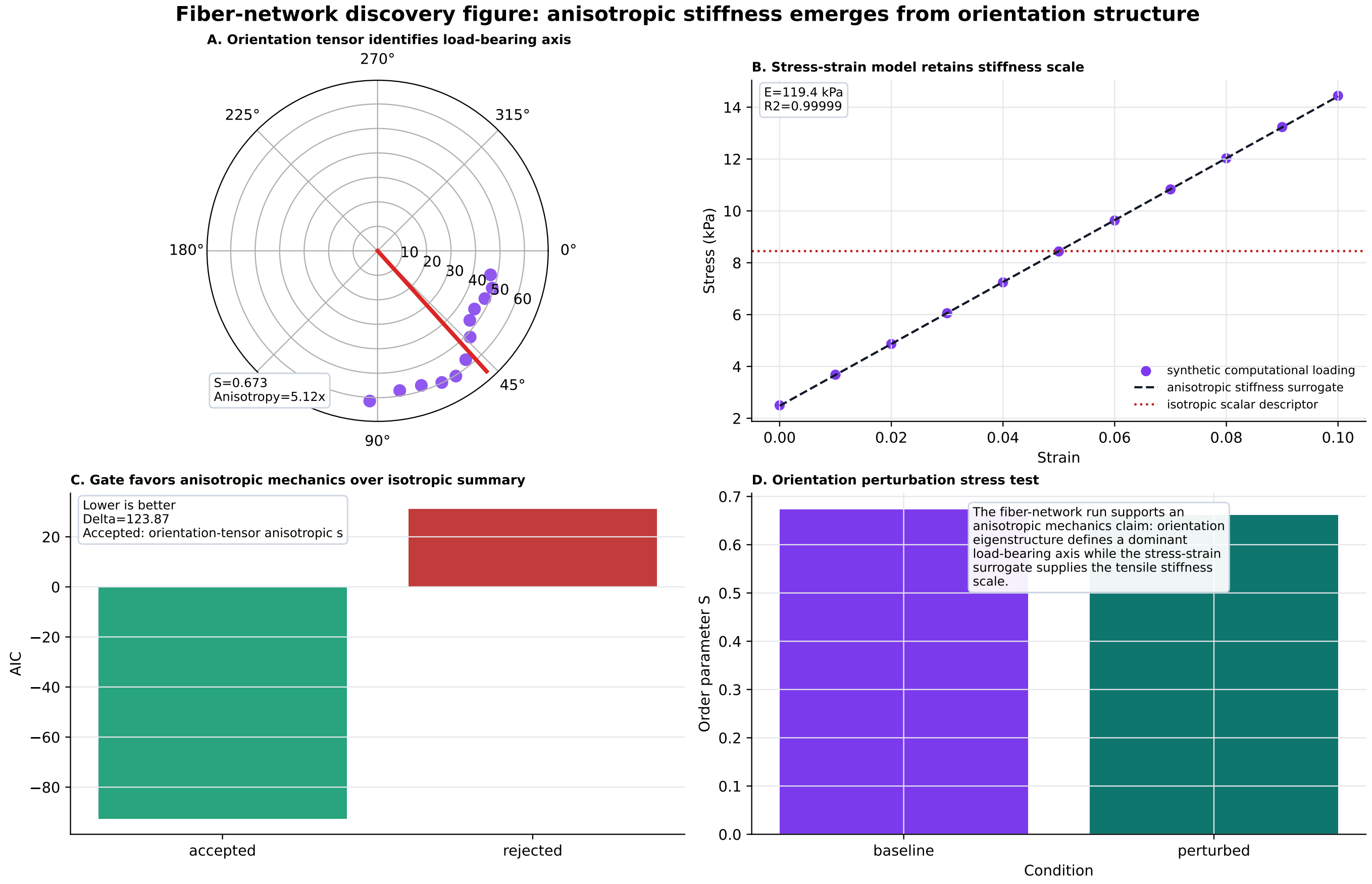


**Figure S4:** Fiber-network numerical mechanics evidence from deterministic synthetic inputs. The figure records the orientation tensor and principal axis, the linear stress–strain fit, the accepted anisotropic stiffness surrogate, the rejected isotropic fiber-count descriptor, the AIC gate, and the orientation-perturbation stress test.



\end{document}